\documentclass[journal,comsoc]{IEEEtran}
\usepackage[T1]{fontenc}% optional T1 font encoding

\ifCLASSINFOpdf
\else
\fi
\usepackage{amsmath}
\usepackage[cmintegrals]{newtxmath}
\usepackage{subfig}
\usepackage{capt-of}
\usepackage{mathtools}
\usepackage{siunitx}
\usepackage[figurename=Figure]{caption}
\usepackage{scrextend}
\usepackage{xspace}
\usepackage{xcolor}
\usepackage{multirow}
\usepackage{wrapfig}

\usepackage{amssymb}

\usepackage{amsthm}
\usepackage{thmtools}
\usepackage{mathbbol}
\usepackage{algorithm}      % algorithm
\usepackage[noend]{algpseudocode} % algorithm
\usepackage{booktabs}
\usepackage[export]{adjustbox}
\usepackage{url}

\usepackage{bm}
\usepackage{tikz}
\usepackage{hyperref}
\usepackage{lineno}
\definecolor{lime}{HTML}{A6CE39}
\DeclareRobustCommand{\orcidicon}{%
	\begin{tikzpicture}
	\draw[lime, fill=lime] (0,0) 
	circle [radius=0.16] 
	node[white] {{\fontfamily{qag}\selectfont \tiny ID}};
	\draw[white, fill=white] (-0.0625,0.095) 
	circle [radius=0.007];
	\end{tikzpicture}
	\hspace{-3mm}
}

\foreach \x in {A, ..., Z}{%
	\expandafter\xdef\csname orcid\x\endcsname{\noexpand\href{https://orcid.org/\csname orcidauthor\x\endcsname}{\noexpand\orcidicon}}
}

\newcounter{algsubstate}
\renewcommand{\thealgsubstate}{\alph{algsubstate}}

\newtheorem{theorem}{Theorem}[section]

\newtheorem{remark}[theorem]{Remark}
\newtheorem*{theorem*}{Theorem}

\newcommand{\fig}{{Fig.}\@\xspace}

\newcommand{\eqn}{{Eq.}\@\xspace}

\newcommand{\tabref}[1]{\mbox{Table~\ref{#1}}}
\newcommand{\figref}[1]{\mbox{Fig.~\ref{#1}}}
\newcommand{\secref}[1]{\mbox{Section~\ref{#1}}}

\makeatletter
\DeclareRobustCommand\onedot{\futurelet\@let@token\@onedot}
\def\@onedot{\ifx\@let@token.\else.\null\fi\xspace}
\def\eg{\emph{e.g}\onedot}
\def\ie{\emph{i.e}\onedot}

\def\etal{\emph{et al}\onedot}
\makeatother

\newcolumntype{L}[1]{>{\raggedright\arraybackslash}p{#1}}
\newcolumntype{C}[1]{>{\centering\arraybackslash}p{#1}}
\newcolumntype{R}[1]{>{\raggedleft\arraybackslash}p{#1}}

\begin{document}
\title{Learning to Minimize the Remainder in\\Supervised Learning}

\author{
    Yan~Luo\orcidA{}~\IEEEmembership{Student Member,~IEEE,},
    Yongkang~Wong\orcidB{}~\IEEEmembership{Member,~IEEE,}
	Mohan~S.~Kankanhalli\orcidC{}~\IEEEmembership{Fellow,~IEEE}, and
	Qi~Zhao\orcidD{}~\IEEEmembership{Member,~IEEE}%
	\IEEEcompsocitemizethanks{Manuscript received August 15, 2021;  revised January 11, 2022; accepted February 20, 2022. 
    	This research was funded in part by the NSF under Grants 1908711 and 1849107, and in part supported by the National Research Foundation, Singapore under its Strategic Capability Research Centres Funding Initiative. Any opinions, findings and conclusions or recommendations expressed in this material are those of the author(s) and do not reflect the views of National Research Foundation, Singapore.
    	The associate editor coordinating the review of this manuscript and approving it for publication was xxx xxx.
		Corresponding author: Q. Zhao (email: qzhao@cs.umn.edu).
		
		Y.~Luo and Q.~Zhao are with the Department of Computer Science and Engineering, University of Minnesota
		(email: luoxx648@umn.edu and qzhao@cs.umn.edu).
		Y.~Wong and M.~Kankanhalli are with the School of Computing, National University of Singapore
		(email: yongkang.wong@nus.edu.sg and mohan@comp.nus.edu.sg).
	}
}

% \markboth{A SUBMISSION TO IEEE TRANSACTIONS ON MULTIMEDIA}%
\markboth{Preprint}%
{Luo~\MakeLowercase{\etal}: Learning to Minimize Remainder in Supervised Learning}



\maketitle

\begin{abstract}

The learning process of deep learning methods usually updates the model's parameters in multiple iterations.
Each iteration can be viewed as the first-order approximation of Taylor's series expansion. The remainder, which consists of higher-order terms, is usually ignored in the learning process for simplicity.
This learning scheme empowers various multimedia-based applications, such as image retrieval, recommendation system, and video search.
Generally, multimedia data (\eg images) are semantics-rich and high-dimensional, hence the remainders of approximations are possibly non-zero.
In this work, we consider that the remainder is informative and study how it affects the learning process.
To this end, we propose a new learning approach, namely gradient adjustment learning (GAL), to leverage the knowledge learned from the past training iterations to adjust vanilla gradients, such that the remainders are minimized and the approximations are improved.
The proposed GAL is model- and optimizer-agnostic, and is easy to adapt to the standard learning framework.
It is evaluated on three tasks, \ie image classification, object detection, and regression, with state-of-the-art models and optimizers.
The experiments show that the proposed GAL consistently enhances the evaluated models, whereas the ablation studies validate various aspects of the proposed GAL.
The code is available at \url{https://github.com/luoyan407/gradient_adjustment.git}.

\end{abstract}

\begin{IEEEkeywords}
	Supervised learning, deep learning, remainder, gradient adjustment.
\end{IEEEkeywords}

% \fi
\IEEEpeerreviewmaketitle

\section{Introduction}
Multimedia applications are the systems that aim to deal with a variety of types of media \cite{Xia_TMM_2012,Ding_TMM_2015,Zhan_TMM_2018,Cho_TMM_2019,Xu_TMM_2020}, such as image, text, etc. Specifically, image classification \cite{He_CVPR_2016,Krizhevsky_NIPS_2012,Tan_ICML_2019} and object detection \cite{He_CVPR_2017,Ren_NIPS_2015} are common components for processing image data.
One of the major challenges is that the image data are semantics-rich and high-dimensional at a large scale \cite{Deng_CVPR_2009,Lin_ECCV_2014}.
Therefore, how to efficiently learn the mapping between images and ground-truth labels is crucial.
Specifically, a learning process
consists of multiple iterations where the parameters of a model are updated by minimizing the scalar parameterized objective function.
Given some training samples and a loss function, each of the training iteration performs a first-order approximation, that is, the Taylor's series expansion omitting higher-order terms \cite{Boyd_CUP_2004}.
\figref{fig:teaser} illustrates the approximation.
Briefly, given a sample $x_{t}, y_{t}$, the loss $\ell_{\sigma}(z_{t})$ is iteratively minimized by subtracting the term $\nabla_{z_{t}}\ell^{\top}_{\sigma}(z_{t})\Delta z_{t}$ while discarding the remainder $r(z_{t})$.
Gradient descent is a simple yet effective solution that uses the gradients to expand the approximation.

\begin{figure}[!t]
	\centering
	\includegraphics[width=1\linewidth]{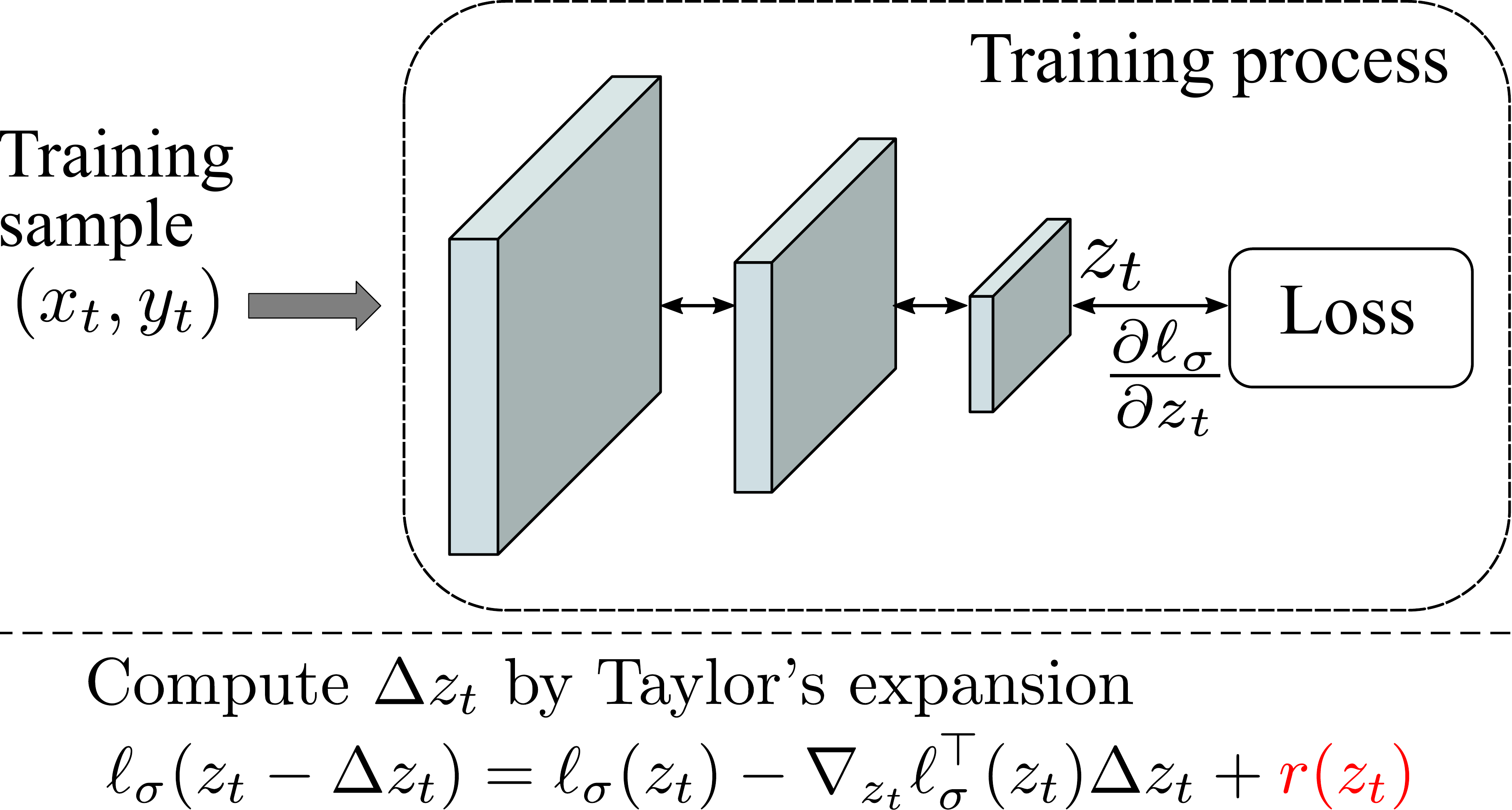}
% 	\vspace{-1ex}
	\caption{\label{fig:teaser}
	    Illustration of the problem of minimizing the remainder $r(z_{t})$ (highlighted in \textcolor{red}{red}), which is usually ignored.
	    Here, $\Delta z_{t} =\hat{\eta}\frac{\partial \ell_{\sigma}}{\partial z}$ for simplicity during the standard learning process.
	    As $r(z_{t})$ is possibly not zero in real-world learning tasks, this work studies how to learn to minimize $r(z_{t})$ by adjusting $\Delta z_{t}$ and its influence during the learning process. Following the convention, $\ell_{\sigma}(\cdot, \cdot)=\ell(\sigma(\cdot), \cdot)$ represents an activation function followed by a loss function. $\ell_{\sigma}(z_{t}, y_{t})$ is further simplified to $\ell_{\sigma}(z_{t})$. %, \eg softmax loss.
	% (\cdot)=\ell(\sigma(\cdot))
    % 	Illustration of ignoring the remainder $r(z_{t})$ in the approximation during the standard learning process, where $r(z_{t})$ is possibly non-zero in computer vision task. In this work, instead of being ignored, the remainder is used for learning to adjust gradients to improve the approximation.
    	}
\end{figure}

% YL: backup caption
% Illustration of ignoring the remainder $r(z_{t})$ in the approximation during the standard learning process, where $r(z_{t})$ is possibly non-zero in computer vision task. In this work, instead of being ignored, the remainder is used for learning to adjust gradients to improve the approximation.
%, instead of being ignored in the standard learning process. 
    % 	$\ell_{\sigma}(z_{t})=\ell(\sigma(z_{t}),y_{t})$ where $\ell$ is the loss function and $\sigma$ is an activation function, \eg the softmax function in the classification task.

However, the remainder that is left in each training iteration is possibly non-zero.
The reasons are three-fold in terms of the problem nature, learning framework, and model generalizability.
Firstly, the diversity of task-dependent semantics and high dimensionality of image data form the learning problems that are difficult to find an approximation with zero remainder.
Secondly, the stochastic process is proven to be helpful for preventing the learning process from overfitting \cite{Bottou_SIAM_2018,Robbins_AMS_1951} and is widely-used in computer vision tasks \cite{Deng_CVPR_2009,Lin_ECCV_2014}. Inevitably, the approximation in the stochastic process could be affected by the underlying noise distribution \cite{Zhu_ICML_2019}.
Lastly, although deep learning techniques \cite{He_CVPR_2016,Krizhevsky_NIPS_2012,Tan_ICML_2019} have achieved remarkable success, the generalizability of models still has room for improvement in producing a better approximation that is with smaller remainder using a variety of labeled images.

In this work, we study the remainder in three tasks, namely, image classification, object detection, and regression. The remainder is informative and could be helpful for improving the learning process. Thus, we aim to minimize the remainder that is difficult to compute and study how it affects the learning process. 
To this end, we propose a learning approach, named gradient adjustment learning (GAL), to leverage the knowledge learned from the past learning steps to adjust the current gradients so the remainder can be minimized. 

The advantages of formulating the minimization of the remainder as a learning problem are two-fold. 
Firstly, instead of limiting to the observed samples at each iteration, the proposed GAL has a broader view on the correlation between all seen samples till the current iteration and the resulting remainders.
%  by observing and learning from the occurred learning steps
Secondly, the remainder which contains all higher-order terms is informative. So, it is a good indicator to gauge if the adjusted gradient better fits the approximation than the vanilla gradient.
However, it is challenging to predict a gradient adjustment vector as the prediction is a continuous real value, instead of discrete labels.
The expected precision is remarkably higher than the one in the classification task, as the values of gradients are sensitive yet decisive to the learning process.
To solve this problem, we devise the proposed GAL to determine how much adjustment will take place, 
which is easy to work with any network model, \eg~multi-layer perceptron (MLP).
Since the optimization process is a mini-ecosystem and gradient works closely with the optimization methods, we investigate the efficacy of the proposed GAL with several state-of-the-art models and optimizers in image classification, object detection, and regression tasks.
The main contributions are as follows.
\begin{itemize}
    \item We propose a novel learning approach, named gradient adjustment learning (GAL), which learns to adjust vanilla gradients for minimizing the remainders of approximations in the learning process.
    We provide the theoretical analysis of the generalization bound and the error bound of the proposed learning approach.
    The proposed approach is model- and optimizer-agnostic.
    \item We propose a safeguard mechanism with a conditional update policy (\ie by verifying the update using the adjusted gradient) to guarantee that the adjusted gradient would lead to an effective descent.
    \item We conduct comprehensive experiments and analyses on CIFAR-10/100~\cite{Krizhevsky_TR_2009}, ImageNet~\cite{Deng_CVPR_2009}, MS COCO~\cite{Lin_ECCV_2014}, Boston housing \cite{Harrison_JEEM_1978}, diabetes \cite{Dua_Report_2019}, and California housing \cite{Pace_SPL_1997}. 
    The experiments show that the proposed GAL demonstrably improves the learning process.
\end{itemize}

\section{Related Work}

\noindent \textbf{Optimization Methods}. 
Stochastic optimization methods often use gradients to update the model parameters \cite{Hinton_RMSProp_2012,Kingma_arXiv_2014,Liu_ICLR_2020O,Luo_ICLR_2018,Robbins_AMS_1951,Zhuang_Neurips_2020}.
In deep learning, stochastic gradient descent (SGD)~\cite{Robbins_AMS_1951} is an influential and practical optimization method.
It takes the anti-gradient as the parameters' update for the descent, based on the first-order approximation \cite{Boyd_CUP_2004}. 
Along the same line, several first- and second-order methods are devised to guarantee convergence to local minima under certain conditions~\cite{Carmon_SIAM_2018,Jin_ICML_2017,Reddi_ICAIS_2018}. 
Nevertheless, these methods are computationally expensive and not feasible for learning settings with large-scale data.
In contrast, adaptive methods, such as Adam~\cite{Kingma_arXiv_2014}, RMSProp~\cite{Hinton_RMSProp_2012}, and Adabound~\cite{Luo_ICLR_2018}, show remarkable efficacy in a broad range of problems~\cite{Hinton_RMSProp_2012,Kingma_arXiv_2014,Luo_ICLR_2018}. 
Zhang~\etal propose an optimization method that wraps an arbitrary optimization method as a component to improve the learning stability \cite{Zhang_NIPS_2019}.
\cite{Andrychowicz_NIPS_2016,Chen_ICML_2017,Ji_IJCAI_2019} learn an optimizer to adaptively compute the step length for updating the models on synthetic or small-scale datasets.
% In this work, we would inspect the interplay between the adjusted gradients yielded by the proposed GAL and these optimization methods.
These methods are contingent on vanilla gradients to update a model. 
In this work, we conduct a study to show how adjusted gradients influence the learning process.
% , in comparison to vanilla gradients.

\noindent \textbf{Gradient-based Methods}.
Given the training data and corresponding ground-truth, a gradient is computed by encoding the task-dependent semantics.
Gradient is crucial in the back-propagation, which enables the learning process to update models' weights such that the loss is minimized \cite{Rumelhart_Nature_1986}.
Gradient-based methods have been proven in modern deep learning models \cite{He_CVPR_2016,Krizhevsky_NIPS_2012,Tan_ICML_2019,Tan_ICML_2021,Dosovitskiy_ICLR_2021,Tolstikhin_arXiv_2021}, which serve as backbones to facilitate a broad range of multimedia applications \cite{Xia_TMM_2012,Ding_TMM_2015,Cho_TMM_2019,Xu_TMM_2020,Cong_TMM_2012,Yadati_TMM_2014,Bu_TMM_2014,Zhang_TMM_2014,Cho_TMM_2015,Zhang_TMM_2019,Li_IJCV_2020}.
% This makes the gradients versatile in a broad range of tasks, other than the aforementioned optimization methods~\cite{Kingma_arXiv_2014,Hinton_RMSProp_2012,Luo_ICLR_2018,Robbins_AMS_1951}.

Except for updating models' weights, gradients are versatile in regulating or regularizing the learning process, e.g., gradient alignment \cite{Li_WACV_2020,Luo_ECCV_2020}, searching for adversarial perturbation \cite{Goodfellow_ICLR_2015}, sharpness minimization \cite{Foret_ICLR_2021}, making decision for choosing hyperparameters \cite{Pedregosa_ICML_2016}, etc.
Specifically, Lopez-Paz and Ranzato propose a gradient episodic memory method that alleviates catastrophic forgetting in continual learning by maintaining the gradient for the update to fit with memory constraints~\cite{Lopez_NIPS_2017}. 
The gradient is aligned to improve the agreement between the knowledge learned from the completed training steps and the new information being used for updating the model~\cite{Luo_TPAMI_2019}. 
In transfer learning, gradients computed by multiple source domains are combined to minimize the loss on the target domain~\cite{Li_WACV_2020}. 
%Moreover, Yong~\etal propose a normalization method that centralizes the gradients to have zero mean~\cite{Yong_arXiv_2020}. 
The proposed GAL is model-agnostic and thus can benefit these applications.

\begin{figure*}[!t]
	\centering
% 	\begin{minipage}[t]{0.47\textwidth}
%         \centering
%         \vspace{0pt}
%         \includegraphics[width=1\linewidth]{fig/workflow_base}
%         \caption{Standard learning paradigm}
%         \label{fig:frame_a}%
%     \end{minipage} \hfill
%     \begin{minipage}[t]{0.47\textwidth}
%         \centering
%         \vspace{0pt}
%         \includegraphics[width=1\linewidth]{fig/workflow_gal}
%         \caption{Proposed gradient adjustment learning (GAL)}
%         \label{fig:frame_b}%
%     \end{minipage}
    \includegraphics[width=1\linewidth]{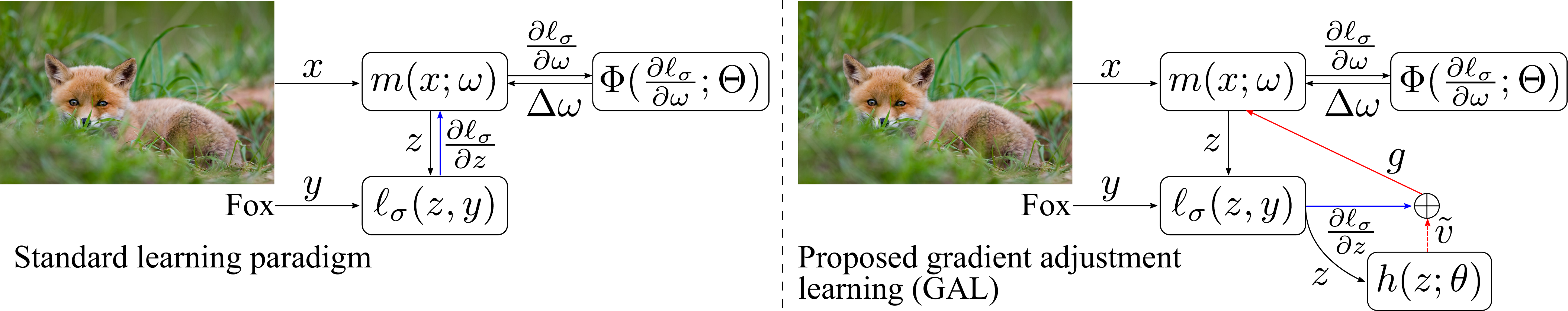}
	\caption{\label{fig:framework}
    	Illustration of standard and proposed learning paradigm. 
    % 	Illustration of (a) a standard learning paradigm and (b) the proposed gradient adjustment learning (GAL). 
    	Note that the proposed learning paradigm is model- and optimizer-agnostic. 
    	If $h(z;\theta)$ always outputs $\bm{0}$, the proposed learning paradigm is reduced to the standard learning paradigm.
    % 	The proposed GAL only takes place during training. 
	}
\end{figure*}

% YL: backup figure data
% \subfigure[Standard learning paradigm]{%
% 		\includegraphics[width=0.45\linewidth]{fig/workflow_base} 
% 		\label{fig:frame_a}%
% 	}
% 	\hfill
% 	\subfigure[Proposed gradient adjustment learning (GAL)]{%
% 		\includegraphics[width=0.45\linewidth]{fig/workflow_gal}%
% 		\label{fig:frame_b}%
% 	}%

\noindent \textbf{Remainder of Approximations}.
% Approximation theory is the branch of mathematics which studies the process of approximating general functions \cite{Achieser_CC_2013,Timan_Elsevier_2014}.
% % an important branch of mathematics. 
% It is helpful when an exact form of problems is difficult to obtain, \eg computer vision tasks.
% Specifically, evaluating the remainder, \ie~the difference between the approximation and the exact form, is common and fundamental in numerical analysis \cite{Berz_RC_1998,Milne_JRNBS_1949,Stancu_MC_1963,Stancu_SIAM_1964}.
% % For example, Stancu studies how to find an integral expression of certain formulas of Taylor type and the remainder of certain interpolation and approximation by Bernstein polynomial formulas \cite{Stancu_SIAM_1964}. 
% These methods purely depend on the mathematical rules and their variable space is low-dimensional, \eg 2-dimensional in \cite{Stancu_SIAM_1964}.
% In computer vision tasks, there is no exact mapping between input and output, whereas the mapping is a deterministic function in \cite{Berz_RC_1998,Milne_JRNBS_1949,Stancu_MC_1963,Stancu_SIAM_1964}.
% % The corresponding theoretical properties of the remainder have been discussed in \cite{Berz_RC_1998}.
% This work provides a new perspective on the remainder of Taylor approximations, 
% which considers remainder as a learning problem on large-scale and diverse data.
Approximation theory is the branch of mathematics which studies the process of approximating general functions \cite{Achieser_CC_2013,Timan_Elsevier_2014}.
For the exact mapping problem with deterministic functions, there are a considerable number of works that study and evaluate the remainder in low-dimensional variable spaces \cite{Berz_RC_1998,Milne_JRNBS_1949,Stancu_MC_1963,Stancu_SIAM_1964}.
However, there is no exact mapping between the input and output in computer vision tasks, where the input image is in a high-dimensional space \cite{Deng_CVPR_2009,Lin_ECCV_2014}.
This makes it difficult to exactly compute the remainder.
As a result, the remainders of approximations are ignored for the sake of simplicity in the learning process \cite{He_CVPR_2016,Krizhevsky_NIPS_2012,Tan_ICML_2019}.
This work is the first to study the effect of minimizing the remainder as a learning problem on large-scale data.
% in the learning process. 
% To this end, modeling the remainder is viewed as a learning problem on large-scale and diverse data in this work.

\section{Problem Formalization}
\label{sec:problem}

% In this section, we first review the standard learning paradigm under supervised learning setting.
% Then, we formally introduce the problem we aim to solve in this work.

Without loss of generality, we consider the standard classification problem where the formulation can be adapted to other learning problems with minor modifications.
Given a training set {\small $D=\{(x_{i},y_{i})|1\le i \le N\}$},
where $x_{i} \in \mathcal{X}$ is the data and $y_{i} \in \{0,1\}^{d}$ is the corresponding ground-truth, 
\ie~$d$ dimensional binary labels, 
a learnable model $m: \mathcal{X} \xrightarrow[]{\omega} \mathbb{R}^{d}$ with parameters $\omega$ is optimized to minimize the loss $\ell$.
% Specifically, the input to $m$ is feature $z\in \mathbb{R}^{d}$ and the output of $m$ is 
According to the empirical risk minimization principle~\cite{Vapnik_TNN_1999}, it can be written as
\begin{align}
    \underset{\omega}{\text{minimize}}\ \frac{1}{|D|} \sum_{(x_{t},y_{t})\in D}^{} \ell\big(\sigma(m(x_{t}; \omega)), y_{t}\big)
\end{align}
where $|D|$ is the cardinality of $D$ and $\sigma: \mathbb{R}^{d} \xrightarrow[]{\omega} [0,1]^{d}$ is an activation function, \eg~softmax layer.
%, \eg~softmax function, in the classification tasks~\cite{Krizhevsky_NIPS_2012,He_CVPR_2016,Tan_ICML_2019}. 

The problem of design and training of $m(\cdot; \omega)$ has been extensively studied \cite{He_CVPR_2016,Krizhevsky_NIPS_2012,Tan_ICML_2019}, and it is not the focus of this work. 
Instead, we focus on the loss w.r.t. the discriminative features $z$, which is the output of $m(\cdot; \omega)$.
Let $\ell_{\sigma}(z)$ denote $\ell(\sigma(z),y)$ for simplicity.
By doing Taylor series expansion,
\begin{align}
\ell_{\sigma}(z_{t}-\Delta z_{t}) = \ell_{\sigma}(z_{t}) - \nabla_{z_{t}} \ell_{\sigma}^{\top}(z_{t})\Delta z_{t} + r(z_{t})
\label{eqn:expansion}
\end{align}
where the loss remainder {\small $r(z_{t})=o(\Delta z_{t})$} is the higher order term w.r.t. $\Delta z_{t}$.
The second term, {\small $\nabla_{z_{t}} \ell_{\sigma}^{\top}(z_{t})\Delta z_{t}$}, is the directional derivative at $z_{t}$ in the direction $\Delta z_{t}$. 
Mathematically, it is difficult to compute higher order derivatives therein for $o(\Delta z_{t})$.
Therefore, maximizing the margin between $\ell_{\sigma}(z_{t})$ and $\ell_{\sigma}(z_{t}-\Delta z_{t})$, which is equivalent to convergence enhancement, is challenging.
% it is difficult to compute $o(\Delta z_{t})$ to maximize the margin between $\ell_{\sigma}(z_{t})$ and $\ell_{\sigma}(z_{t}-\Delta z_{t})$, which is equivalent to enhance the convergence. 
Moreover, $(z_{t},y_{t})$ follows some stochastic process and would vary with the iterations. 
Different $(z_{t},y_{t})$ pair may contribute unevenly to the learning process.

% To enhance the learning process, it is ideal to find a $\Delta z_{t}$ that can minimize the difference between $\ell_{\sigma}(z_{t}-\Delta z_{t})$ and a local minimal loss $\ell_{\sigma}(z_{t}^{*})$, where $z_{t}^{*}$ is the desired variable to achieve $\ell_{\sigma}(z_{t}^{*})$.
% \begin{align}
% \underset{}{\text{minimize}}\ \ell_{\sigma}(z_{t}-\Delta z_{t})-\ell_{\sigma}(z_{t}^{*})
% \label{eqn:idl_obj_remainder}
% \end{align}
% This is equivalent to minimize $r(z_{t})$ at $z_{t}^{*}$, as $\nabla_{z_{t}^{*}} \ell_{\sigma}(z_{t}^{*}) = 0$. 
% Conventionally, $\Delta z_{t}$ is determined by the first-order partial derivatives.

\figref{fig:framework} (left) shows a standard learning approach where $r(z_{t})$ is omitted. 
We denote $\Phi(\cdot;\Theta)$ as an optimizer 
% (\eg~SGD, Adam, RMSProp) 
with a set of hyperparameters $\Theta$ such as learning rate, momentum, weight decay, etc. 
The key step in this optimization process is that loss function $\ell$ takes the prediction $\hat{y}=\sigma(z)$ and the ground-truth $y$ as input to compute the gradient $\frac{\partial \ell_{\sigma}}{\partial z} = \nabla_{z}\ell_{\sigma}(\hat{y}, y)$. 
According to the chain rule, the gradient $\frac{\partial \ell_{\sigma}}{\partial \omega}$ is computed by $\frac{\partial \ell_{\sigma}}{\partial z} \frac{\partial z}{\partial \omega}$. 
Next, $\Delta \omega = \Phi(\frac{\partial \ell_{\sigma}}{\partial \omega};\Theta)$ is computed to update $\omega$.

In the standard learning approach, the gradient $\frac{\partial \ell_{\sigma}}{\partial z}$ is mathematically computed and can be considered as a local choice over observed inputs $(x,y)$ at each iteration.
Making a local choice at each step can be viewed as a greedy strategy and may find less-than-optimal solutions~\cite{Black_DADS_2005}.
% As similar to most methods based on locally optimal strategy, 
% not lead to a global optimum or even a local optimum \cite{Black_DADS_2005}.
% In contrast, we consider how to adjust the gradients to enhance the convergence of the training process as a learning problem, as shown in \fig~\ref{fig:frame_b} and \ref{fig:frame_c}.
In contrast, this work adjusts the gradient by an adjustment module which aims to minimize the remainder (as shown in \figref{fig:framework} (right)).
Correspondingly, the adjustment can be viewed as an addition of two vectors, where one is the vanilla gradient and the other is the vector generated by the adjustment module.
A geometric interpretation is shown in \figref{fig:frame_c}.

%%
%% ----------  section split  ---------- %% 
%%

\section{Gradient Adjustment Learning}

In this section, we first describe the gradient adjustment mechanism in a supervised learning framework.
Then, the training process of the proposed gradient adjustment module is detailed.
Finally, we discuss its theoretical properties. %of the gradient adjustment. 

%-------------------------------------------------------
\subsection{Gradient Adjustment in Learning Process}

Here, we introduce the integration of the proposed GAL into the standard learning approach.
We first define a gradient adjustment module $h(\cdot;\theta)$ (see \fig \ref{fig:framework}), 
% \eg~a DNN module, 
which aims to model the correlation between the adjustment at point $z$ and the corresponding loss remainder $r$, \ie
\begin{align}
v = h(z;\theta), \ \ v\in \mathbb{R}^{d}
\label{eqn:predict}
\end{align}

Different from a classifier that predicts a confidence score between 0 and 1, the proposed GAL learns to predict a gradient adjustment vector which tends to be small, sophisticated, and subtle.
% vary dramatically in both samples and dimensions. \NOTE{why}
To curb its volatility, which could overwhelm the gradient and ruin the learning process, we apply a normalization with $l^{2}$ norm to adaptively scale it to coincide with the gradient, \ie
\begin{align}
	\tilde{v}=\alpha \Big| \frac{\partial \ell_{\sigma}}{\partial z} \Big| v/ |v|
\label{eqn:nml}
\end{align}
where $\alpha\in [0,1]$ is a scalar that constrains the relative strength of adjustment referencing to the magnitude of $\frac{\partial \ell_{\sigma}}{\partial z}$. 
$\alpha=0$ implies no adjustment will be performed.
The normalized feature $\tilde{v}$ is added to the computed vanilla gradient and is used as the input to the optimizer for model updating
\begin{align}
	g=\frac{\partial \ell_{\sigma}}{\partial z}+\tilde{v}
\label{eqn:adj}
\end{align}
The gradient adjustment module $h(\cdot;\theta)$ can be any type of DNNs, such as MLP, CNN, or RNN. 
As the computed adjustment is possibly negative in some dimensions, we remove the final activation layer (\eg~softmax layer). 

\begin{figure}[!t]
	\centering
	\includegraphics[width=0.75\linewidth]{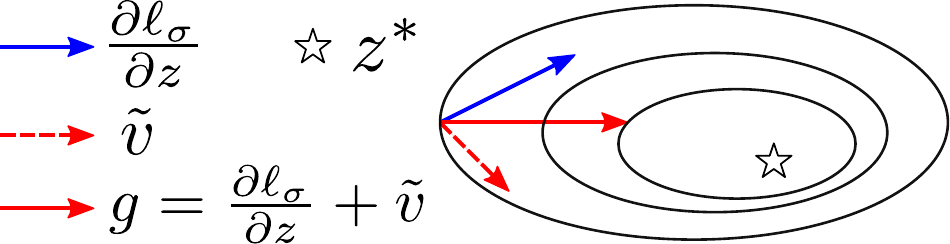}%
	\caption{\label{fig:frame_c}
    	Geometric interpretation of the proposed GAL. The adjustment is performed by a vector addition operation.
    	}
\end{figure}

Lines 7--10 in Algorithm~\ref{alg:gl} are the conditional update policy that compute update $\Delta \omega$ based on the relationship between $\ell_{\sigma} (z-\tilde{\eta}g)$ and $\ell_{\sigma} (z)$.
Here, $\tilde{\eta}$ is the tentative learning rate and $\ell_{\sigma} (z-\tilde{\eta}g)$ is the tentative loss (detailed in \secref{sec:modelTraining}).
% which will be introduced in following sections in detail.
Checking $\ell_{\sigma} (z-\tilde{\eta}g) \le \ell_{\sigma} (z)$ is able to detect if $g$ is not a good fit to reduce the loss.
In this case, we alternatively use vanilla gradient for update.
% In this case, we alternatively use the gradient generated by the standard process.
This can be regarded as a safeguard mechanism to verify whether the adjusted gradient $g$ leads to an effective descent.

\begin{algorithm}[!t]
	\caption{Gradient Adjustment Learning}\label{alg:gl} 
	\begin{algorithmic}[1]
		\State \textbf{Input}: $D$, $m(\cdot;\omega)$, $h(\cdot;\theta)$, $\Phi(\cdot;\Theta)$ (learning rate $\eta \in \Theta$), magnitude ratio $\alpha \in [0,1]$, adaptive scalar $\beta$ so $\tilde{\eta} = \beta\eta$
		\For {Each pair $(x,y)\in D$}
		\State $z=m(x;\omega),\ \hat{y}=\sigma(z)$ 
		\State $\frac{\partial \ell_{\sigma}}{\partial z} = \nabla_{z}\ell_{\sigma} (z)$
		\State Predict gradient adjustment $v = h(z;\theta)$
		\State Adjust gradient $g=\frac{\partial \ell_{\sigma}}{\partial z}+\tilde{v}, \ \ \tilde{v}=\alpha|\frac{\partial \ell_{\sigma}}{\partial z}| v/ |v|$
%		\State Compute the update $\Delta \omega = \Phi(g\frac{\partial z}{\partial \omega};\Theta)$
		\If {$\ell_{\sigma} (z-\tilde{\eta}g)\le \ell_{\sigma} (z)$}
		\State $\Delta \omega = \Phi(g\frac{\partial z}{\partial \omega};\Theta)$
		\Else
		\State $\Delta \omega = \Phi(\frac{\partial \ell_{\sigma}}{\partial z}\frac{\partial z}{\partial \omega};\Theta)$
		\EndIf
		\State Update parameters $\omega \leftarrow \omega - \Delta \omega$
		\State Minimize the remainder (the objective (\ref{eqn:obj_remainder})):
%			$\hat{\ell}=\ell(f(x+\eta g))+\eta\nabla_{z}\ell (\hat{y}, y)^{\top}g$
% 		\begin{algsubstates}
			\State Compute $\frac{\partial r}{\partial v}$
			\State Compute the update $\Delta \theta = \Phi(\frac{\partial r}{\partial v} \frac{\partial v}{\partial \theta};\Theta)$
			\State Update the adjustment module's parameters 
			\State $\theta \leftarrow \theta - \Delta \theta$
% 		\end{algsubstates}
%		\State compute gradient $\frac{\partial \ell(\sigma(z-\beta\eta g), y)}{\partial z} = \nabla_{z}\hat{\ell}(g)$
%		\State compute the update $\Delta \theta = \Phi(\frac{\partial \hat{\ell}(g)}{\partial z} \frac{\partial z}{\partial \theta};\Theta)$
%		\State update adjustment learner's parameters $\theta \leftarrow \theta - \Delta \theta$
		\EndFor
	\end{algorithmic} 
\end{algorithm}

%--------------------------------------
\subsection{Adjustment Module Training}
\label{sec:modelTraining}

% In the objective (\ref{eqn:expansion}), 
% $z^{*}$ is unknown hence it is difficult to minimize this objective in practice. 
% Instead, we minimize the loss remainder at $z$ so that the learning process is enhanced to reach $z^{*}$.
As discussed in \secref{sec:problem}, the remainder $r(z_{t})$ in \eqn~(\ref{eqn:expansion}) is difficult to estimate in practice.
% Given \eqn~(\ref{eqn:expansion}), we know how the remainder plays its role in the approximation.
% As the remainder is difficult to be estimated directly, 
However, the remainder can be modeled with the other three terms in the equation. 
So, this turns the estimation to a learning problem, \ie
\begin{gather}
    \underset{\theta}{\text{minimize}}\ |r(z_{t})|, \label{eqn:obj_remainder}\\
    r(z_{t}) = \ell_\sigma(z_{t}-\tilde{\eta} g)-\ell_\sigma(z_{t})+\tilde{\eta}\nabla \ell_\sigma(z_{t})^{\top}g, \label{eqn:loss_remainder}
\end{gather}
where $\ell_\sigma(z_{t}-\tilde{\eta} g)$ is the tentative loss and $\tilde{\eta}=\beta\eta$ is the  tentative learning rate. 
Briefly, the tentative loss is used to evaluate whether the adjusted gradient $g$ is better than $\frac{\partial \ell}{\partial z}$. 
Although $z_{t}-\tilde{\eta} g$ is a decision condition, it still needs a learning rate to fit into the gradient descent scheme. 
A straightforward way of doing it is by using a hyperparameter $\beta$ as weight on the learning rate $\eta$ for parameters update. 
In this way, $\tilde{\eta}$ is adaptive to $\eta$. 
Note that $|r(z_{t})|$ is minimized in objective (\ref{eqn:obj_remainder}) rather than $r(z_{t})$. 
This is because the prediction is subtle and it is possible to overfit or underfit the remainder.

From \eqn~(\ref{eqn:predict}) and (\ref{eqn:nml}), it can be seen that $g$ is a function of $\theta$. The objective (\ref{eqn:obj_remainder}) provides information for adjusting the gradient in a direction that reduces the remainder of first-order Taylor approximation.

\begin{figure*}[!t]
	\centering
	\begin{tabular}{ccccc}
	\footnotesize{Gradient descent} & \footnotesize{RMSProp} & \footnotesize{Adam} & \footnotesize{Lookahead} & \footnotesize{Adabound} \\
	\subfloat{%
		\includegraphics[trim=0 0 0 15,clip,width=0.174\linewidth]{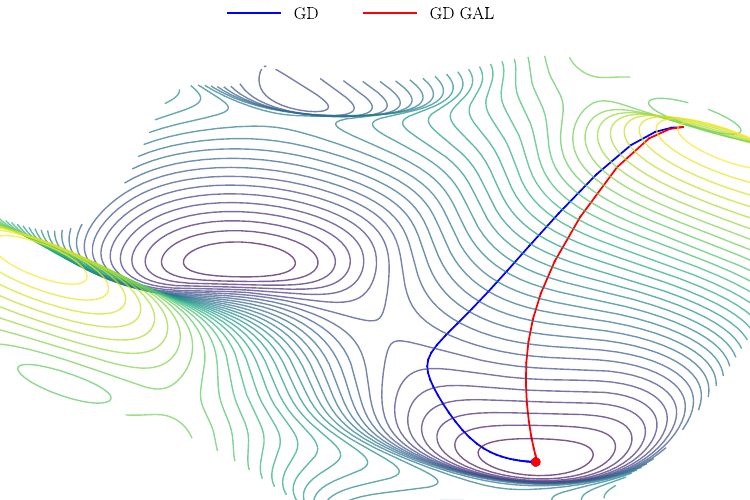}%
	}%
	&
	\subfloat{%
		\includegraphics[trim=0 0 0 15,clip,width=0.174\linewidth]{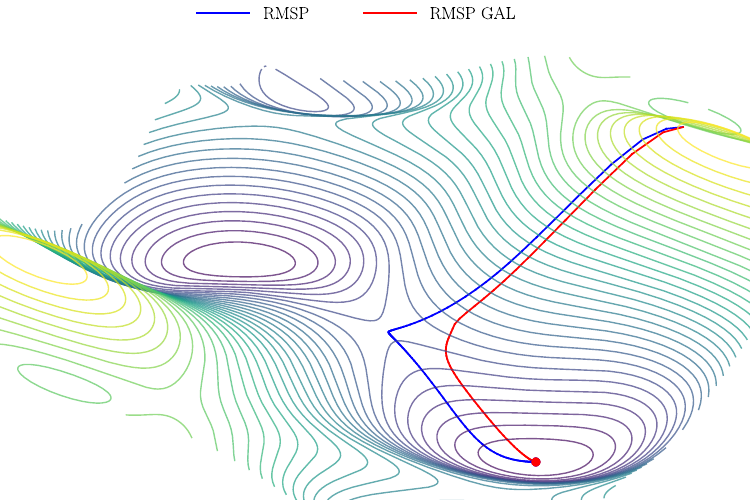}%
	}
	&
	\subfloat{%
		\includegraphics[trim=0 0 0 15,clip,width=0.174\linewidth]{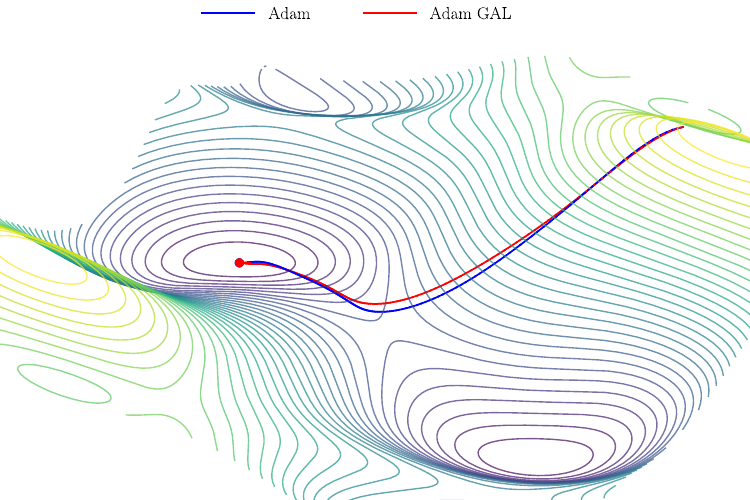}%
	}
	&
	\subfloat{%
		\includegraphics[trim=0 0 0 15,clip,width=0.174\linewidth]{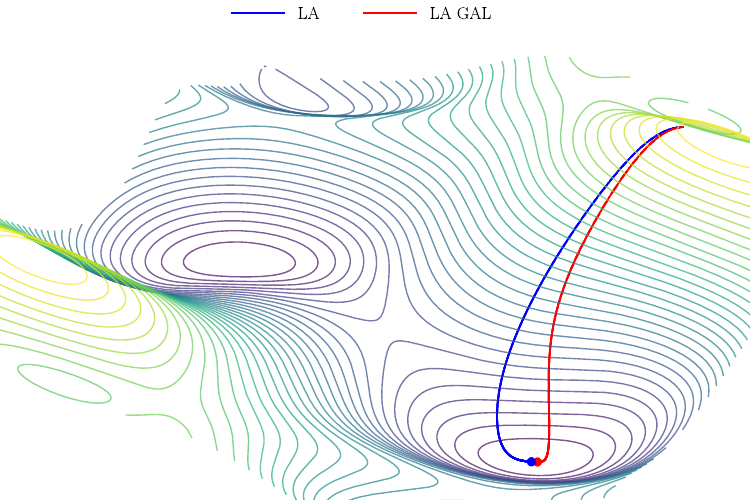}%
	}%
	&
	\subfloat{%
		\includegraphics[trim=0 0 0 15,clip,width=0.174\linewidth]{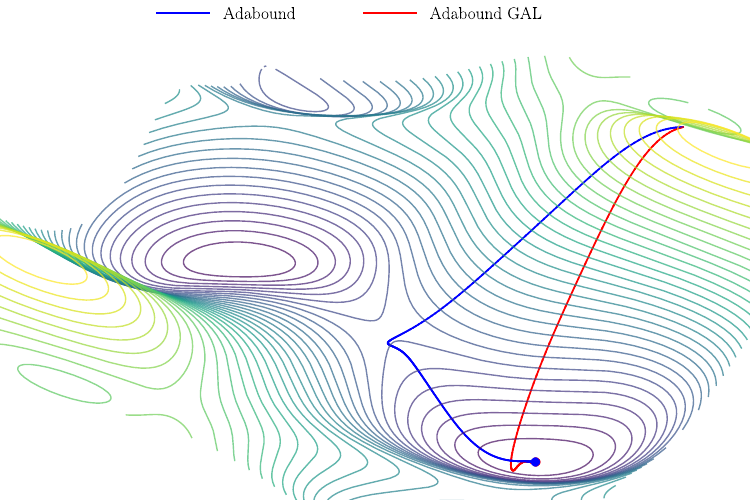}%
	}  \\
	\subfloat{%
		\includegraphics[width=0.174\linewidth]{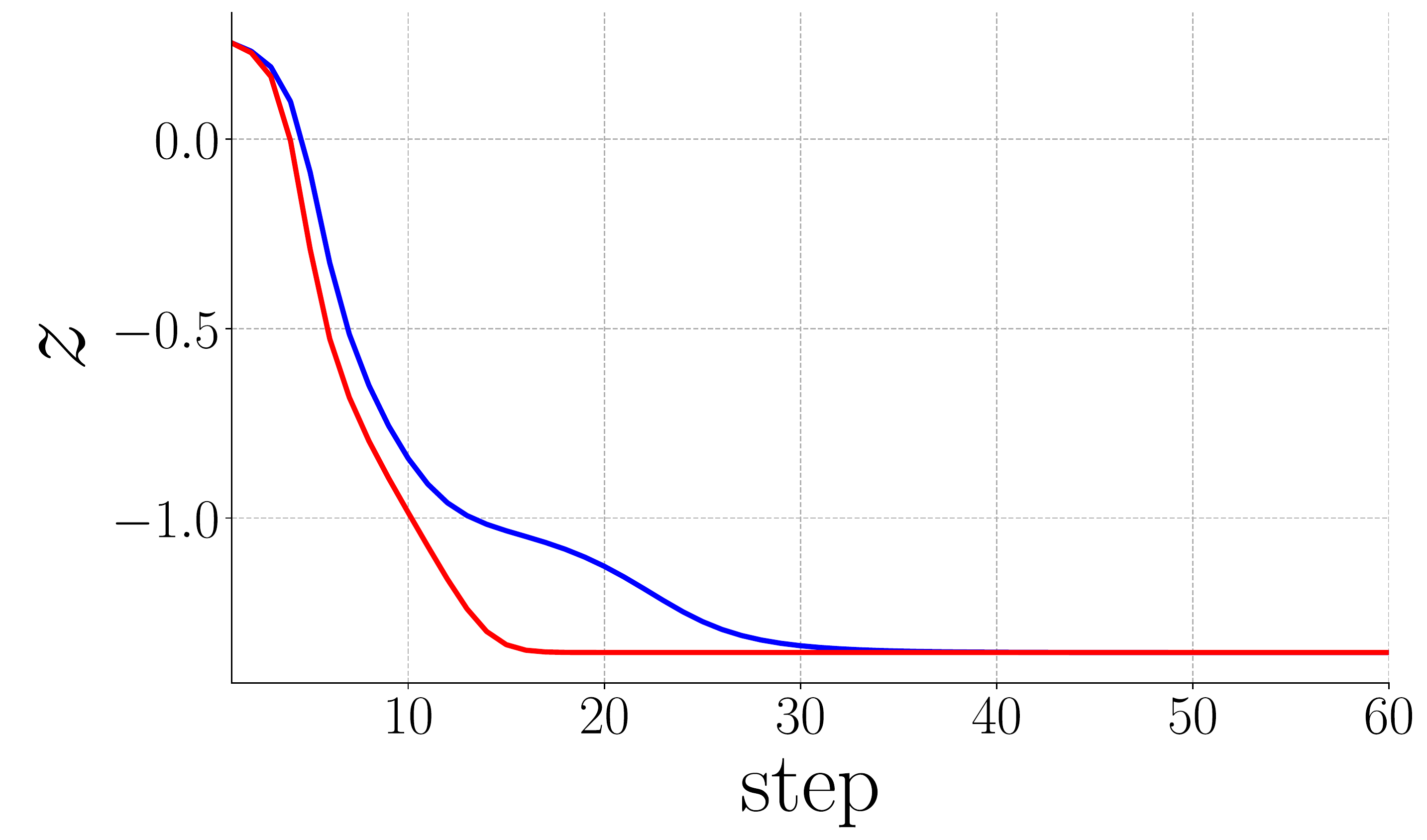}%
	} 
	&
	\subfloat{%
		\includegraphics[width=0.174\linewidth]{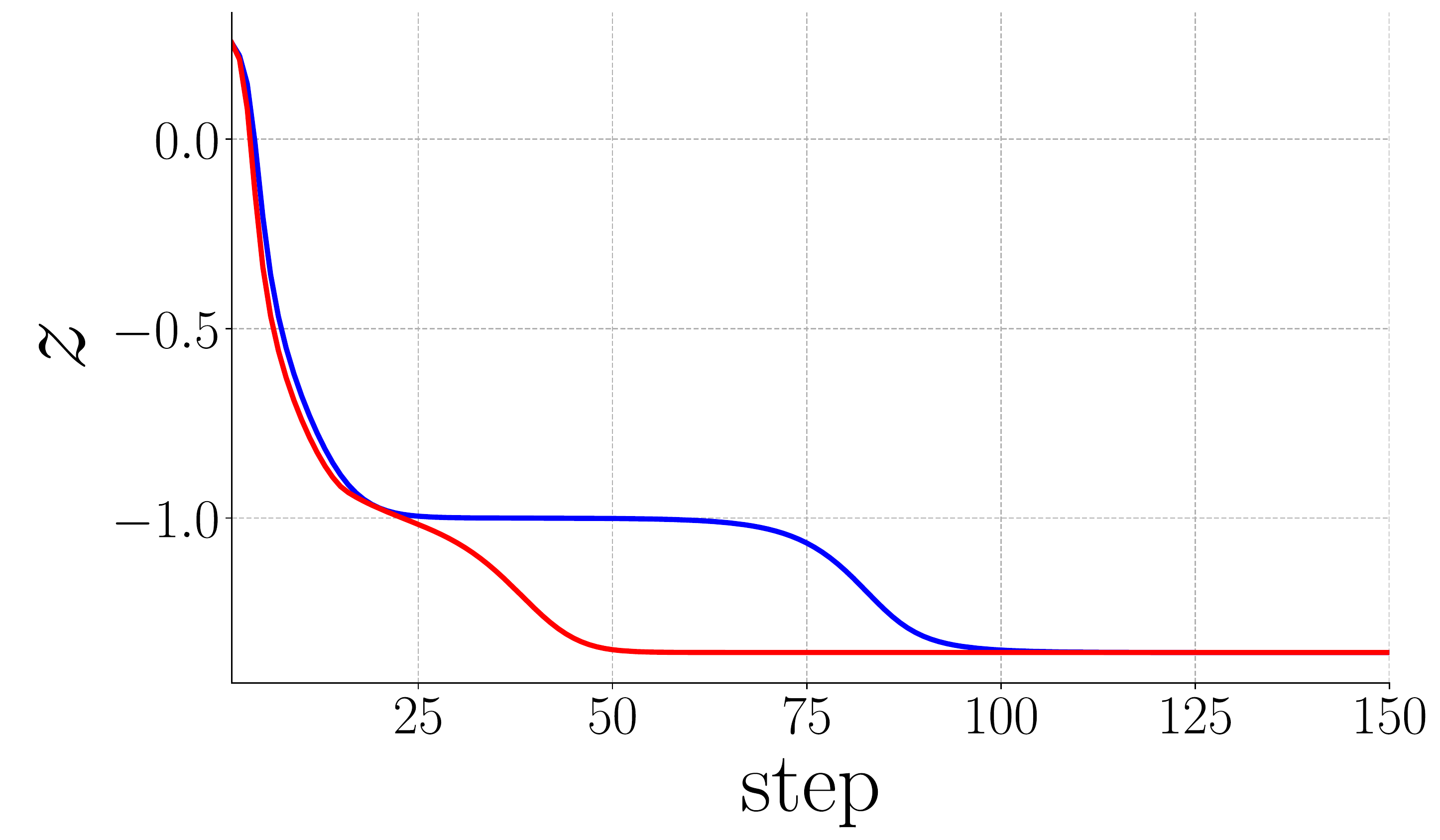}%
	}
	&
	\subfloat{%
		\includegraphics[width=0.174\linewidth]{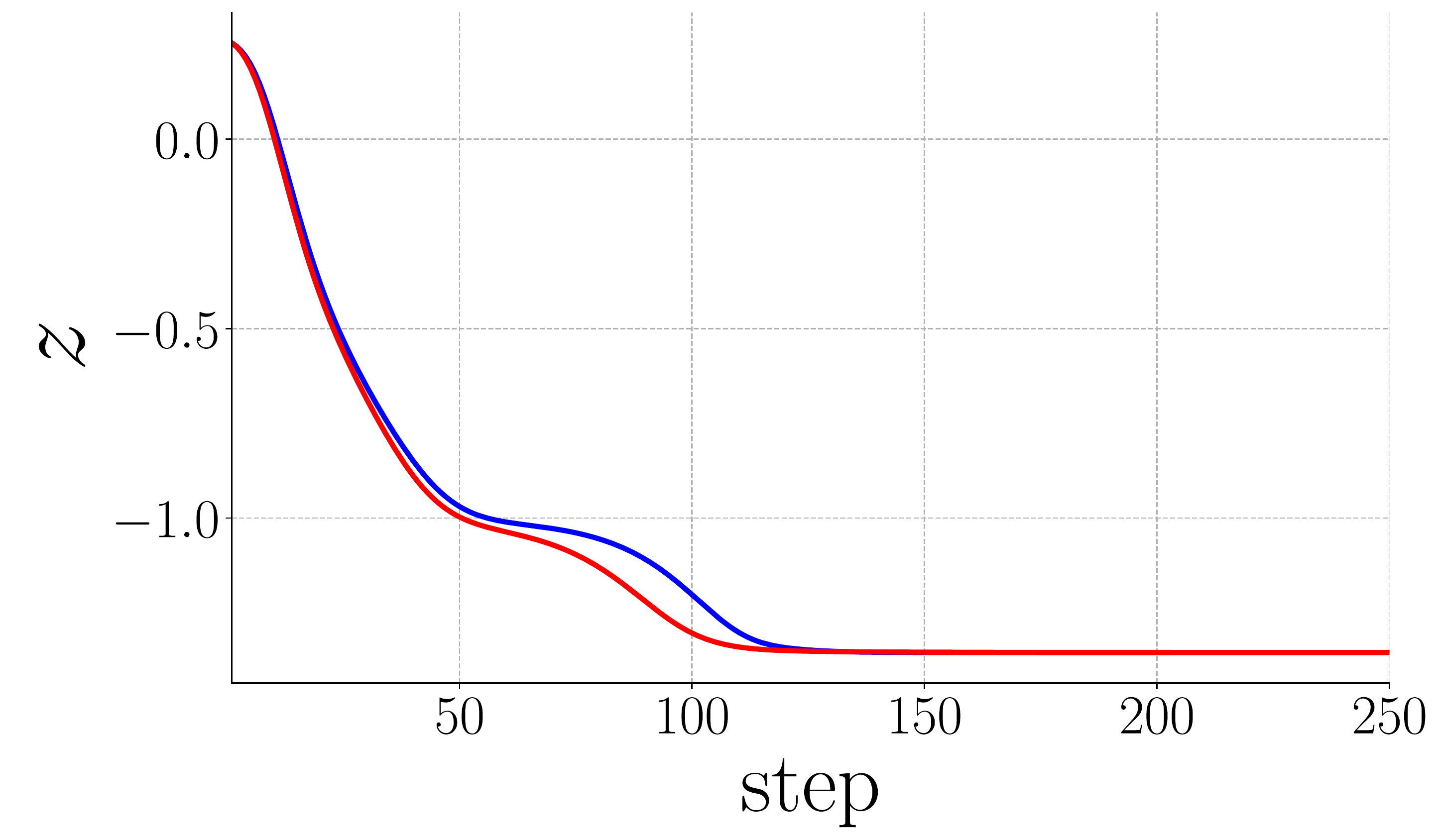}%
	}
	&
	\subfloat{%
		\includegraphics[width=0.174\linewidth]{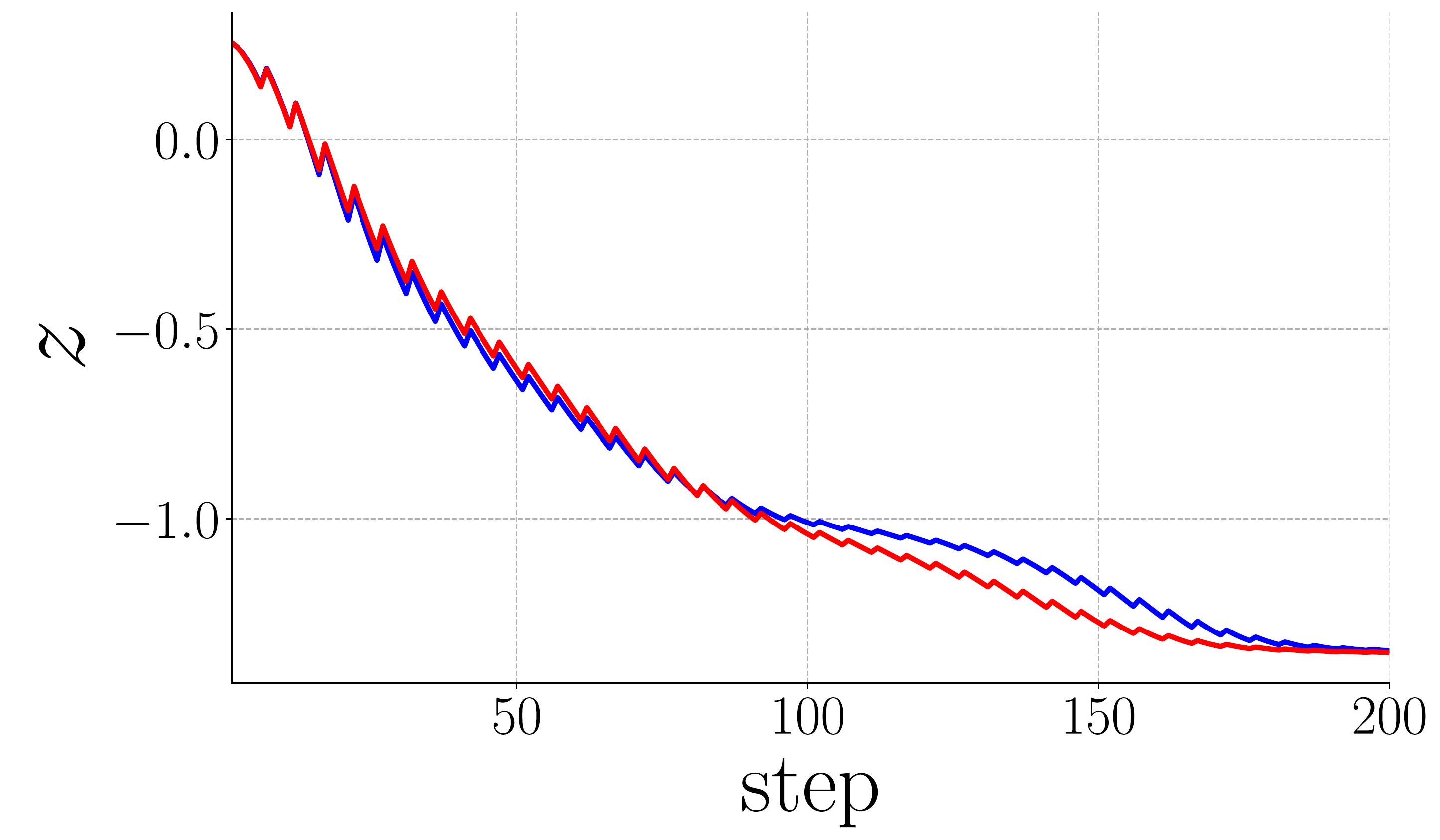}%
	}%
	&
	\subfloat{%
		\includegraphics[width=0.174\linewidth]{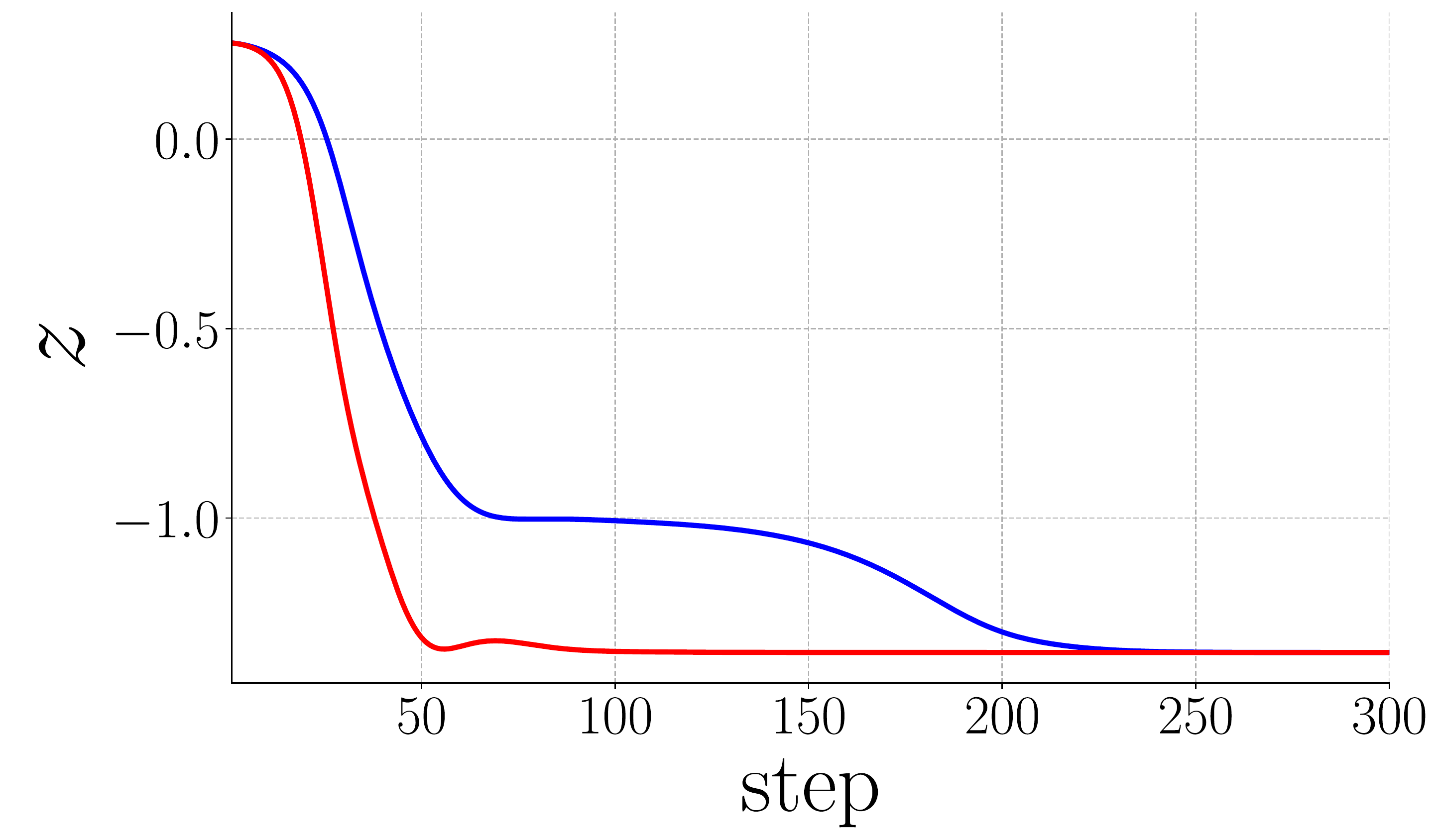}%
	} \\
	\end{tabular}
% 	\vspace{-2ex}
	\caption{\label{fig:illustration}
	    Illustrations of the effect of GAL (\textcolor{red}{red} path/curves) on convergence with various optimizers, in comparison with the standard process (\textcolor{blue}{blue} path/curves). The top row are convergence paths, while the bottom are the corresponding loss curves. 
	    The problem is publicly available in \cite{Luo_TPAMI_2019}.
	    }
\end{figure*}

% \begin{figure*}[!t]
% 	\centering
% 	\subfigure[Gradient descent]{%
% 		\includegraphics[trim=0 0 0 15,clip,width=0.19\linewidth]{fig/illu/gd}%
% 	}%
% 	\hfill
% 	\subfigure[RMSProp]{%
% 		\includegraphics[trim=0 0 0 15,clip,width=0.19\linewidth]{fig/illu/rmsp}%
% 	}
% 	\hfill
% 	\subfigure[Adam]{%
% 		\includegraphics[trim=0 0 0 15,clip,width=0.19\linewidth]{fig/illu/adam}%
% 	}
% 	\hfill%
% 	\subfigure[Lookahead]{%
% 		\includegraphics[trim=0 0 0 15,clip,width=0.19\linewidth]{fig/illu/lagd}%
% 	}%
% 	\hfill
% 	\subfigure[Adabound]{%
% 		\includegraphics[trim=0 0 0 15,clip,width=0.19\linewidth]{fig/illu/adabound}%
% 	}
% %	\vspace{-5px}
% 	\\
% 	\subfigure[Gradient descent]{%
% 		\includegraphics[width=0.19\linewidth]{fig/illu/gd_z}%
% 	} 
% 	\hfill
% 	\subfigure[RMSProp]{%
% 		\includegraphics[width=0.19\linewidth]{fig/illu/rmsp_z}%
% 	}
% 	\hfill
% 	\subfigure[Adam]{%
% 		\includegraphics[width=0.19\linewidth]{fig/illu/adam_z}%
% 	}
% 	\hfill%
% 	\subfigure[Lookahead]{%
% 		\includegraphics[width=0.19\linewidth]{fig/illu/laadam_z}%
% 	}%
% 	\hfill
% 	\subfigure[Adabound]{%
% 		\includegraphics[width=0.19\linewidth]{fig/illu/adabound_z}%
% 	}%
% 	\caption{\label{fig:illustration}
% 	    Illustrations of the effect of GAL (\textcolor{red}{red} path/curves) on convergence with various optimizers, in comparison to the standard process (\textcolor{blue}{blue} path/curves). The top row are convergence paths, while the bottom are the corresponding loss curves. 
% 	    The problem is publicly available in \cite{Luo_TPAMI_2019}.
% 	    }
% \end{figure*}
%----------------------------------
\subsection{Theoretical Properties}

This section presents the learning guarantee and remainder error bound for the GAL problem. %We relegate all proofs to the supplementary material.
% the appendix.
% 
For simplicity, we denote $h(z;\theta)$ as $h(z)$. Let $v^{*}\in \mathbb{R}^{d}$ be the target adjustment so $z-\tilde{\eta}(\nabla f(z)+v^{*})=z^{*}$. As the gradient adjustment vector is usually small, we assume there exist $a, b\in \mathbb{R}$ so that $v,v^{*}\in [a,b]^{d}\subseteq \mathbb{R}^{d}$, and $z$ is drawn i.i.d. according to the unknown distribution $\mathcal{D}$ and $v^{*}=h^{*}(z)$ where $h^*(\cdot)$ is the target labeling function. Moreover, we follow the problem setting in \cite{Mohri_MIT_2012} to restrict the loss function to be the $\ell_{p}$ loss ($p\ge 1$) for generalization bound.
GAL can be considered as a variant of regression problem that finds the hypothesis $h: \mathbb{R}^{m}\rightarrow [a,b]^{d}$ in a set $\mathcal{H}$ with small generalization error w.r.t. $h^{*}$, \ie
\begin{align*}
R_{\mathcal{D}}(h)=E_{z\sim \mathcal{D}} \big[\ell_{p}(h(z), h^{*}(z)) \big].
\end{align*}
In practice, as $\mathcal{D}$ is unknown, we use the empirical error for approximation over samples in dataset $D$, \ie
\begin{align*}
\hat{R}_{D}(h)=\frac{1}{|D|}\sum_{i=1}^{|D|}\ell_{p} \big( h(z_{i}), v^{*}_{i} \big),
\end{align*}

\begin{theorem}[Generalization Bound of GAL]
	Denote $\mathcal{H}$ as a finite hypothesis set. Given $v,v^{*}\in [a,b]^{d}$, for any $\delta>0$, with probability at least $1-\delta$, the following inequality holds for all $h\in \mathcal{H}$:
	\begin{align*}
	|R_{\mathcal{D}}(h) - \hat{R}_{D}(h)| \le  \sqrt[\leftroot{-3}\uproot{3}p]{d(b-a)^{p}} \cdot \sqrt{\frac{\log|\mathcal{H}|+\log\frac{2}{\delta}}{2|D|}}
	\end{align*}
	\label{thm:gb}
\end{theorem}
\begin{proof}
	The proof sketch is similar to the classification generalization bound provided in \cite{Mohri_MIT_2012}.
	First, as $\ell_{p}(v,v^{*})=(\sum_{i}^{d}|v_{i}-v^{*}_{i}|^p)^{\frac{1}{p}}\le (d(b-a)^{p})^{\frac{1}{p}}$, we know $\ell_{p}$ is bounded by $(d(b-a)^{p})^{\frac{1}{p}}$. Then, by the union bound, given an error $\xi$, we have
	\begin{align*}
	Pr[\sup_{h\in \mathcal{H}}|R(h)-\hat{R}(h)| > \xi] \le \sum_{h\in \mathcal{H}}^{} Pr[|R(h)-\hat{R}(h)|> \xi].
	\end{align*}
	By Hoeffding's bound, we have
	\begin{align*}
	\sum_{h\in \mathcal{H}}^{} Pr[|R(h)-\hat{R}(h)|> \xi] \le 2|\mathcal{H}|\exp \left(-\frac{2|D|\xi^2}{(d(b-a)^{p})^{\frac{2}{p}}}\right).
	\end{align*}
	Due to the probability definition, $2|\mathcal{H}|\exp (-\frac{2|D|\xi^2}{(d(b-a)^{p})^{\frac{2}{p}}}) = \delta$. Considering $\xi$ is a function of other variables, we can rearrange it as 
	$\xi=(d(b-a)^{p})^{\frac{1}{p}}\sqrt{\frac{\log|\mathcal{H}|+\log\frac{2}{\delta}}{2|D|}}$.
	%	\begin{align*}
	%	\xi=m^{\frac{1}{p}}\sqrt{\frac{\log|H|+\log\frac{2}{\delta}}{2|D|}}.
	%	\end{align*}
	Since we know $Pr[|R(f)-\hat{R}(f)| > \xi]$ is with probability at most $\delta$, it can be inferred that $Pr[|R(f)-\hat{R}(f)| <= \xi]$ is at least $1-\delta$.
	It completes the proof.
\end{proof}

\begin{remark}
	Theorem~\ref{thm:gb} supports the general intuition that more training data should produce better generalization, which is aligned with conventional learning problems, \eg classification and regression \cite{Mohri_MIT_2012}. Furthermore, distinct from conventional learning problems, the range of gradient adjustments and the dimension could affect the generalization bound.
\end{remark}

% \begin{proof}
% 	The proof sketch is similar to the classification generalization bound provided in \cite{Mohri_MIT_2012}.
% 	First, as $\ell_{p}(v,v^{*})=(\sum_{i}^{d}|v_{i}-v^{*}_{i}|^p)^{\frac{1}{p}}\le (d(b-a)^{p})^{\frac{1}{p}}$, we know $\ell_{p}$ is bounded by $(d(b-a)^{p})^{\frac{1}{p}}$. Then, by the union bound, given an error $\xi$, we have
% 	\begin{align*}
% 	Pr[\sup_{h\in \mathcal{H}}|R(h)-\hat{R}(h)| > \xi] \le \sum_{h\in \mathcal{H}}^{} Pr[|R(h)-\hat{R}(h)|> \xi].
% 	\end{align*}
% 	By Hoeffding's bound, we have
% 	\begin{align*}
% 	\sum_{h\in \mathcal{H}}^{} Pr[|R(h)-\hat{R}(h)|> \xi] \le 2|\mathcal{H}|\exp \left(-\frac{2|D|\xi^2}{(d(b-a)^{p})^{\frac{2}{p}}}\right).
% 	\end{align*}
% 	Due to the probability definition, $2|\mathcal{H}|\exp (-\frac{2|D|\xi^2}{(d(b-a)^{p})^{\frac{2}{p}}}) = \delta$. Considering $\xi$ is a function of other variables, we can rearrange it as 
% 	$\xi=(d(b-a)^{p})^{\frac{1}{p}}\sqrt{\frac{\log|\mathcal{H}|+\log\frac{2}{\delta}}{2|D|}}$.
% %	\begin{align*}
% %	\xi=m^{\frac{1}{p}}\sqrt{\frac{\log|H|+\log\frac{2}{\delta}}{2|D|}}.
% %	\end{align*}
% 	Since we know $Pr[|R(f)-\hat{R}(f)| > \xi]$ is with probability at most $\delta$, it can be inferred that $Pr[|R(f)-\hat{R}(f)| <= \xi]$ is at least $1-\delta$.
% 	It completes the proof.
% \end{proof}

\begin{theorem}[Conventional Remainder Error Bound \cite{Nesterov_Springer_2013}]
	Let $f\in C^{1,1}_{L}(\mathbb{R}^{n})$ (\ie $f$ is once continuously differentiable on $\mathbb{R}^{n}$ and its first-order partial derivative is Lipschitz continuous with constant $L$). Then for any $z^{+},\ z\in \mathbb{R}^{n}$ we have
	\begin{align*}
	|f(z^{+})-f(z)-\nabla^{\top}f(z)(z^{+}-z)| \le \frac{L}{2}\|z^{+}-z\|^{2}
	\end{align*}
	\label{thm:nesterov}
\end{theorem}

\begin{theorem}[Revisited Remainder Error Bound]
	Let $f\in C^{1,1}_{L}(\mathbb{R}^{n})$. Given $\tau\in [0,1]$ and any $z^{+},\ z\in \mathbb{R}^{n}$, we denote the minimal angle between vectors $\nabla f(z+\tau(z^{+}-z))-\nabla f(z)$ and $z^{+}-z$ as $\gamma$ and assume the two vectors are non-zero. Then we have
	\begin{small}
	\begin{align*}
	|f(z^{+})-f(z)-\nabla^{\top}f(z)(z^{+}-z)| \le \frac{L}{2}|\cos{\gamma}|\cdot\|z^{+}-z\|^{2}
	\end{align*}
	\end{small}
	\label{thm:error}
\end{theorem}
\begin{proof}
	Similar to the proof of Theorem \ref{thm:nesterov} \cite{Nesterov_Springer_2013}, we use the integral form of the remainder in Taylor's expansion
	% 	\begin{align*}
	% 	& \  |f(z^{+})-f(z)-\nabla^{\top}f(z)(z^{+}-z)| \\
	% 	&= |\int_{0}^{1} (\nabla f(z+\tau(z^{+}-z))-\nabla f(z))^{\top}(z^{+}-z) d\tau| \\
	% 	&\le \int_{0}^{1} |(\nabla f(z+\tau(z^{+}-z))-\nabla f(z))^{\top}(z^{+}-z)| d\tau \ \ \text{(boundedness)} \\
	% 	&= \int_{0}^{1} |\cos{\gamma}|\cdot\|\nabla f(z+\tau(z^{+}-z))-\nabla f(z)\| \cdot \|z^{+}-z\|| d\tau \ \ \text{(cosine similarity)} \\
	% 	&\le |\cos{\gamma}|\int_{0}^{1} \|\nabla f(z+\tau(z^{+}-z))-\nabla f(z)\| \cdot \|z^{+}-z\| d\tau \ \ \text{(Cauchy–Schwarz)} \\
	% 	&\le L\cdot|\cos{\gamma}| \cdot \|z^{+}-z\|^{2} \int_{0}^{1} \tau d\tau \ \ \text{(L-Lipschitz)} \\
	% 	&= \frac{L}{2}|\cos{\gamma}|\cdot\|z^{+}-z\|^{2}
	% 	\end{align*}
	\begin{align*}
	& \  |f(z^{+})-f(z)-\nabla^{\top}f(z)(z^{+}-z)| \\
	&= |\int_{0}^{1} (\nabla f(z+\tau(z^{+}-z))-\nabla f(z))^{\top}(z^{+}-z) d\tau| \\
	&\le \int_{0}^{1} |(\nabla f(z+\tau(z^{+}-z))-\nabla f(z))^{\top}(z^{+}-z)| d\tau \\
	&= \int_{0}^{1} |\cos{\gamma}|\cdot\|\nabla f(z+\tau(z^{+}-z))-\nabla f(z)\| \cdot \|z^{+}-z\|| d\tau \\
	&\le |\cos{\gamma}|\int_{0}^{1} \|\nabla f(z+\tau(z^{+}-z))-\nabla f(z)\| \cdot \|z^{+}-z\| d\tau \\
	&\le L\cdot|\cos{\gamma}| \cdot \|z^{+}-z\|^{2} \int_{0}^{1} \tau d\tau \\
	&= \frac{L}{2}|\cos{\gamma}|\cdot\|z^{+}-z\|^{2}
	\end{align*}
	It completes the proof.
\end{proof}

% \begin{proof}
%     Similar to the proof of Theorem \ref{thm:nesterov} \cite{Nesterov_Springer_2013}, we use the integral form of the remainder in Taylor's expansion
% 	\begin{align*}
% 	|f(z^{+})-f(z)-\nabla^{\top}f(z)(z^{+}-z)| &= |\int_{0}^{1} (\nabla f(z+\tau(z^{+}-z))-\nabla f(z))^{\top}(z^{+}-z) d\tau| \\
% 	&\le \int_{0}^{1} |(\nabla f(z+\tau(z^{+}-z))-\nabla f(z))^{\top}(z^{+}-z)| d\tau \ \ \text{(boundedness)} \\
% 	&= \int_{0}^{1} |\cos{\gamma}|\cdot\|\nabla f(z+\tau(z^{+}-z))-\nabla f(z)\| \cdot \|z^{+}-z\|| d\tau \ \ \text{(cosine similarity)} \\
% 	&\le |\cos{\gamma}|\int_{0}^{1} \|\nabla f(z+\tau(z^{+}-z))-\nabla f(z)\| \cdot \|z^{+}-z\| d\tau \ \ \text{(Cauchy–Schwarz)} \\
% 	&\le L\cdot|\cos{\gamma}| \cdot \|z^{+}-z\|^{2} \int_{0}^{1} \tau d\tau \ \ \text{(L-Lipschitz)} \\
% 	&= \frac{L}{2}|\cos{\gamma}|\cdot\|z^{+}-z\|^{2}
% 	\end{align*}
% 	It completes the proof.
% \end{proof}

\begin{remark}
	Theorem~\ref{thm:error} shows a tighter error bound of the remainder than the well-known bound in Theorem~\ref{thm:nesterov} \cite{Nesterov_Springer_2013}. It justifies why properly adjusting gradients direction leads to an effective descent. This is a new insight, as compared to Theorem~\ref{thm:nesterov}. Moreover, it indicates the optimal condition from a geometric perspective, that is, if $z^{+}-z$ is perpendicular to $\nabla f(z+\tau(z^{+}-z))-\nabla f(z)$,  the remainder error bound is zero. This is feasible as $z^{+}-z$ is liable to be small in terms of magnitude and $\nabla f(z+\tau(z^{+}-z))-\nabla f(z)$ will not vary dramatically with $\tau\in [0,1]$.
	In addition, the theorem provides some guideline to design \eqn \ref{eqn:nml}, which forces the adjustment module to find a direction instead of a vector itself for stability.
	% learn
\end{remark}

%We follow AlexNet to devise the regressor $h(\cdot;\theta)$, \ie, a convolutional neural network (CNN) which mainly consists of convolutional and linear layers. The notable difference between the regressor and the CNN for classification is that the regressor aims to predict a gradient-like vector where the values could be negative. Considering the gradient could vary intensively over dimensions, we use the learnable cube function as the last activation function
%\begin{align}
%\hat{\sigma}(z_{o}) = \alpha z_{o}^{3}
%\end{align}
%where $\alpha$ is learnable and $z_{o}$ is the output of the last layer in the regressor. This function is an element-wise operation.

%----------------------------------------------
\subsection{Adaptivity to Optimization Methods}

To illustrate the effectiveness of the proposed GAL on the optimization process, we employ the 3D problem used in \cite{Luo_TPAMI_2019}, \ie~$z=f(x,y)$, where $x,y,z\in \mathbb{R}$, to visualize the convergence path w.r.t. various optimizers.
\figref{fig:illustration} show the convergence paths (\ie top row) and the corresponding curves of $z$ against steps (\ie bottom row).
Specifically, the blue paths/curves are produced by the standard process, while the red ones are produced by the proposed GAL.
Given the same starting point, the convergence is affected by the problem and optimizers.
The proposed GAL observes the completed convergence steps to learn to adjust the gradients.
The resulting convergence curves show that it finds shortcuts to reach the local minimum efficiently.
Furthermore, \figref{fig:illustration} verifies that the proposed GAL is general in nature and can work with various optimizers.
% The blue path is the convergence path of standard process and the corresponding  curve of $z$ against steps 
% Based on the convergence paths and the corresponding descent curves show in \fig~\ref{fig:illustration},
% the proposed GAL finds a shortcut to reach the local minimum efficiently.
\section{Experiments}
\label{sec:exp}

We comprehensively evaluate the proposed GAL with various models and optimizers. 
Specifically, we conduct experiment on the image classification task~\cite{Luo_ICLR_2018,Zhang_NIPS_2019}, the object detection task~\cite{Carion_ECCV_2020}, and the regression task \cite{Harrison_JEEM_1978,Dua_Report_2019,Pace_SPL_1997}.

\subsection{Datasets}
Following the experimental protocol in \cite{Luo_ICLR_2018,Zhang_NIPS_2019}, we use CIFAR-10/100 \cite{Krizhevsky_TR_2009} and ImageNet \cite{Deng_CVPR_2009} for evaluation on the image classification task. 
Specifically, CIFAR-10 (CIFAR-100) consists of 50,000 32$\times$32 images with 10 (100) classes, while ImageNet has 1000 visual concepts (\ie classes) and provides average 1000 real-world images on each class.
For object detection experiments, we follow the experimental protocol in \cite{Carion_ECCV_2020} to use COCO 2017 \cite{Lin_ECCV_2014} for evaluation.
MS COCO is a large-scale object detection benchmark dataset that consists of 82,783 training images and 40,504 validation images with 80 object categories.
Moreover, three datasets, \ie~Boston housing \cite{Harrison_JEEM_1978}, diabetes \cite{Dua_Report_2019}, and California housing \cite{Pace_SPL_1997}, are used for the regression task. Specifically, Boston housing includes 506 entries and each entry has 14 features, diabetes consists of 442 samples that have 10 features, and California housing has 20640 samples and each sample has 8 features.

\begin{table}[!t]
	\centering
	\caption{\label{tbl:cifar10}
	    Image classification performance on CIFAR-10. The average error and its standard deviation are over three runs. Architecture (100-32-16) is used for GAL and the number of parameters of GAL is 5K.
	   % The top performance is highlighted in bold.
	    }
% 	\vspace{-1em}
	\adjustbox{width=1.0\columnwidth}{
	\begin{tabular}{L{43ex} C{12ex}}
		\toprule
		Model (optimizer)                                & Error (\%)      \\
		\cmidrule(lr){1-1} \cmidrule(lr){2-2}
		PreResNet-110 (Lookahead)~\cite{Zhang_NIPS_2019} & 4.73          \\ 
		DenseNet-121 (Adabound)~\cite{Luo_ICLR_2018}     & 5.00          \\ 
		EfficientNet B0 (-)~\cite{Tan_ICML_2019}         & 1.90          \\
		EfficientNet B1 (SGD)~\cite{Luo_TPAMI_2019}      & 1.91          \\ \midrule
		EffcientNet B1 (SGD) reproduced                  & 1.92$\pm$0.12 \\ 
		EffcientNet B1 (SGD) GAL                         & \textbf{1.84}$\pm$0.06 \\ \midrule
        EffcientNet B1 (Lookahead) reproduced            & 2.01$\pm$0.02 \\ 
		EffcientNet B1 (Lookahead) GAL                   & 1.91$\pm$0.02 \\ \midrule
		EffcientNet B1 (Adabound) reproduced             & 3.15$\pm$0.03 \\ 
		EffcientNet B1 (Adabound) GAL                    & 3.03$\pm$0.01 \\
		\bottomrule	
	\end{tabular}}
\end{table}

% \begin{table*}[!t]
% 	\centering
% 	\caption{Image classification performance on CIFAR. The average error and its standard deviation are over three runs. Architecture (100-32-16) and (256-64-32) is used for GAL on CIFAR-10 and CIFAR-100, respectively. The number of parameters of GAL is 5K on CIFAR-10 and 47K on CIFAR-100.}
% 	\label{tbl:cifar}
% % 	\footnotesize
% 	\begin{tabular}{lcc}
% 		\toprule
% 		Model (optimizer)                                & CIFAR-10      & CIFAR-100      \\
% 		\cmidrule(lr){1-1} \cmidrule(lr){2-2} \cmidrule(lr){3-3}
% 		PreResNet-110 (Lookahead)~\cite{Zhang_NIPS_2019} & 4.73          & 21.63          \\ 
% 		DenseNet-121 (Adabound)~\cite{Luo_ICLR_2018}     & 5.00          & -              \\ 
% 		EfficientNet B0 (-)~\cite{Tan_ICML_2019}         & 1.90          & 11.90          \\
% 		EfficientNet B1 (SGD)~\cite{Luo_TPAMI_2019}      & 1.91          & 11.81          \\ \midrule
% 		EffcientNet B1 (SGD) reproduced                  & 1.92$\pm$0.12 & 11.81$\pm$0.10 \\ 
% 		EffcientNet B1 (SGD) GAL                         & 1.84$\pm$0.06 & 11.37$\pm$0.10 \\ \midrule
%         EffcientNet B1 (Lookahead) reproduced~~~~~~      & 2.01$\pm$0.02 & 11.70$\pm$0.01 \\ 
% 		EffcientNet B1 (Lookahead) GAL                   & 1.91$\pm$0.02 & 11.44$\pm$0.02 \\ \midrule
% 		EffcientNet B1 (Adabound) reproduced             & 3.15$\pm$0.03 & 14.44$\pm$0.06 \\ 
% 		EffcientNet B1 (Adabound) GAL                    & 3.03$\pm$0.01 & 14.12$\pm$0.06 \\
% 		\bottomrule	
% 	\end{tabular}
% \end{table*}

\subsection{Models \& Training Scheme}
In the image classification task, we adopt the state-of-the-art EfficientNet~\cite{Tan_ICML_2019} on CIFAR, and ResNet \cite{He_CVPR_2016} and EfficientNet on ImageNet.
Originally, EfficientNet is trained on Cloud TPU for 350 epochs with batch size of 2048\footnote[1]{https://rb.gy/rz0tus} \cite{Tan_ICML_2019}.
Due to the limitation of computation resources, we follow the training scheme in \cite{Luo_TPAMI_2019} to train EfficientNet models on CIFAR. 
Similarly, we employ a publicly available implementation\footnote[2]{https://github.com/rwightman/pytorch-image-models} to train ResNet and EfficientNet on ImageNet with 8 NVIDIA V100 GPUs with batch size of 320. 
We train the models for 90 epochs \cite{He_CVPR_2016,Zhang_NIPS_2019} to provide comparable results.
In the object detection task,
% we follow the DEtection TRansformer (DETR) \cite{Carion_ECCV_2020} to train the model. 
DEtection TRansformer (DETR) is originally trained with 16 NVIDIA V100 GPUs for 500 epochs \cite{Carion_ECCV_2020}. 
Due to the limitation of computation resources, we follow DETR's suggestion\footnote[3]{https://github.com/facebookresearch/detr} to train the model with 4 NVIDIA 2080 Ti GPUs for 150 epochs. 
We use the same hyperparameters as in \cite{Carion_ECCV_2020}. 
Regarding the optimization methods, the model is trained on CIFAR with SGD, Lookahead~\cite{Zhang_NIPS_2019}, and Adabound~\cite{Luo_ICLR_2018}.
Following \cite{Zhang_NIPS_2019}, Lookahead is wrapped around SGD in the experiments.
The models are trained with RMSProp~\cite{Hinton_RMSProp_2012} on ImageNet.
DETR is trained with AdamW \cite{Loshchilov_ICLR_2018} on MS COCO.
The regression experiments run on CPUs with Adam \cite{Kingma_arXiv_2014}.

\begin{table}[!t]
	\centering
	\caption{\label{tbl:cifar100}
	    Image classification performance on CIFAR-100. The average error and its standard deviation are over three runs. Architecture (256-64-32) is used for GAL and the number of parameters of GAL is 47K.
	    }
% 	\vspace{-1em}
	\adjustbox{width=1.0\columnwidth}{
	\begin{tabular}{L{43ex} C{12ex}}
		\toprule
		Model (optimizer)                                & Error (\%)      \\
		\cmidrule(lr){1-1} \cmidrule(lr){2-2}
		PreResNet-110 (Lookahead)~\cite{Zhang_NIPS_2019} & 21.63          \\ 
		DenseNet-121 (Adabound)~\cite{Luo_ICLR_2018}     & -              \\ 
		EfficientNet B0 (-)~\cite{Tan_ICML_2019}         & 11.90          \\
		EfficientNet B1 (SGD)~\cite{Luo_TPAMI_2019}      & 11.81          \\ \midrule
		EffcientNet B1 (SGD) reproduced                  & 11.81$\pm$0.10 \\ 
		EffcientNet B1 (SGD) GAL                         & \textbf{11.37}$\pm$0.10 \\ \midrule
        EffcientNet B1 (Lookahead) reproduced            & 11.70$\pm$0.01 \\ 
		EffcientNet B1 (Lookahead) GAL                   & 11.44$\pm$0.02 \\ \midrule
		EffcientNet B1 (Adabound) reproduced             & 14.44$\pm$0.06 \\ 
		EffcientNet B1 (Adabound) GAL                    & 14.12$\pm$0.06 \\
		\bottomrule	
	\end{tabular}}
\end{table}

\begin{table*}[!t]
	\centering
	\caption{\label{tbl:imgnet}
	    Image classification performance on ImageNet. The average accuracy and its standard deviation are over three runs.
	   % standard deviation
	   % Architecture (512-256) is used for GAL.
	   Arch (512-128) and (512-256) are used for GAL with ResNet and EfficientNet, respectively.
% 	Note that EfficientNet \cite{Tan_ICML_2019} is trained on Google TPU with 350 epochs at batch size 2049. which is difficult to reproduce. Instead, we use a publicly available GPU implementation\protect\footnotemark[2]. 
	    We use 90 epochs in model training for a fair comparison~\cite{He_CVPR_2016,Zhang_NIPS_2019}.
	    }
% 	\vspace{-1em}
	\adjustbox{width=0.9\textwidth}{
	\begin{tabular}{L{45ex} c L{20ex} C{15ex} C{15ex}}
		\toprule
		Model (optimizer) & & \# of parameters & Top-1 & Top-5 \\
		\cmidrule(lr){1-1} \cmidrule(lr){3-3} \cmidrule(lr){4-4} \cmidrule(lr){5-5} 
		ResNet-50 (SGD) \cite{He_CVPR_2016}          & & ~~23M & 76.15 & 92.87 \\ 
		ResNet-50 (Lookahead) \cite{Zhang_NIPS_2019} & & ~~23M & 75.49 & 92.53 \\ 
		EfficientNet-B2 (RMSProp) 350 epochs \cite{Tan_ICML_2019} & & ~~9.2M & 80.30 & 95.00 \\ \midrule
		ResNet-50 (RMSProp) reproduced       & & 23.5M          & 76.43$\pm$0.02          & 93.05$\pm$0.04 \\
		ResNet-50 (RMSProp) GAL              & & 23.5M + 0.70M  & 76.53$\pm$0.03          & 93.13$\pm$0.05 \\ \midrule
% 		MobileNetV3 (RMSProp) reproduced     & & ~~5.4M         &                &  \\
% 		MobileNetV3 (RMSProp) GAL            & & ~~5.4M + 0.90M &                &  \\ \midrule
		EfficientNet-B2 (RMSProp) reproduced & & ~~9.2M         & 77.93$\pm$0.09 & 93.92$\pm$0.03 \\
		EfficientNet-B2 (RMSProp) GAL        & & ~~9.2M + 0.90M & \textbf{78.10}$\pm$0.06 & \textbf{93.94}$\pm$0.06 \\
		\bottomrule	
	\end{tabular}}
\end{table*}

For the proposed GAL, we employ the MLP, which is simpler than CNN and RNN, throughout this work. 
GAL takes the feature $z \in \mathbb{R}^{d}$ as input and yields the same dimension output for gradient adjustment.
For simplicity, we denote a (N+1)-layer MLP as ({\small $\#_1\!-\!\#_2\!-\cdots-\!\#_N$}).
For example, (100-32-16) indicates that the architecture consists of four linear transformations that have affine matrices in $\mathbb{R}^{10}\times\mathbb{R}^{100}$, $\mathbb{R}^{100}\times\mathbb{R}^{32}$, $\mathbb{R}^{32}\times\mathbb{R}^{16}$, and $\mathbb{R}^{16}\times\mathbb{R}^{10}$. 
We use architectures (100-32-16) on CIFAR-10, (256-64-32) on CIFAR-100, and (512-128/256) on ImageNet. 
Regarding $(\alpha, \beta)$, we use (0.001, 1), (0.01, 1), and (0.01, 10) with SGD, Lookahead, and Adabound, respectively, on CIFAR-10; 
(0.01, 1), (0.001, 5), and (0.01, 10) with SGD, Lookahead, and Adabound, respectively, on CIFAR-100; 
and (0.001, 0.001) on ImageNet, respectively.
For the object detection tasks, we minimize the remainder w.r.t. predicted bounding box features, \ie four floats indicating a box. Correspondingly, we use (64-16), 0.01 and 1 as the arch, $\alpha$ and $\beta$, respectively.
For the regression task, the architectures of the regression models are (100-50), (64-32), and (256-64) on Boston housing, diabetes, and California housing, respectively.
The architectures of the proposed gradient adjustment modules are (16-4), (128-4), and (128-2) on Boston housing, diabetes, and California housing, respectively.
We fix $\alpha=0.001$ and $\beta=0.001$ on all three datasets.

\begin{table*}[!t]
	\centering
	\caption{\label{tbl:mscoco}
	    Object detection performance on MS COCO validation with Faster R-CNN.
	    % Comparison with Faster R-CNN with the standard process and the proposed GAL on the COCO validation set.
	    We follow DETR's suggestion to use 150 epochs in model training \cite{Carion_ECCV_2020}\protect\footnote[3]{}. 
% 	As limited by the available computational resources, instead of using 500 epochs for training \cite{Carion_ECCV_2020}, we follow DETR's suggestion to use 150 epochs. 
	This setting takes approximate 9 days for training DETR-ResNet-50 on a 4-GPU server.
	}
% 	\vspace{-1em}
	\adjustbox{width=1.0\linewidth}{
% 	\begin{tabular}{L{45ex} c L{14ex} c C{14ex} c C{14ex}}
	\begin{tabular}{L{30ex} C{7ex} L{18ex} C{7ex} C{7ex} C{7ex} C{7ex} C{7ex} C{7ex}}
		\toprule
		Model & Epochs & \# of parameters & AP & AP\textsubscript{50} & AP\textsubscript{75} & AP\textsubscript{S} & AP\textsubscript{M} & AP\textsubscript{L} \\
		\cmidrule(lr){1-1} \cmidrule(lr){2-2} \cmidrule(lr){3-3} \cmidrule(lr){4-4} \cmidrule(lr){5-5} \cmidrule(lr){6-6} \cmidrule(lr){7-7} \cmidrule(lr){8-8} \cmidrule(lr){9-9}
		DETR-ResNet-50 \cite{Carion_ECCV_2020} & 500 & 41M & 42.00 & 62.40 &  44.20 & 20.50 & 45.80 & 61.10 \\ 
% 		DETR-DC5 & 41M & 43.30 & 63.10 & 45.90 & 22.50 & 47.30 & 61.10 \\ 
		DETR-ResNet-101 & 500 & 60M & 43.50 & 63.80 & 46.40 & 21.90 & 48.00 & 61.80 \\ 
		\midrule
		DETR-ResNet-50 reproduced & 150 & 41M & 39.13 & 60.03 & 40.94 & 18.30 & 42.51 & 58.62 \\ 
		DETR-ResNet-50 GAL & 150 & 41M+1344 & 39.61 & 60.55 & 41.62 & 18.42 & 42.52 & 59.02 \\ 
% 		DETR-ResNet-101 reproduced & 60M & 41.27 & 62.21 & 43.74 & 19.00 & 45.42 & 60.66 \\
		DETR-ResNet-101 reproduced & 150 & 60M & 40.98 & 61.91 & 43.59 & 19.32 & 45.09 & 60.25 \\
		DETR-ResNet-101 GAL & 150 & 60M+1344 & 41.38 & 62.24 & 43.87 & 20.13 & 45.10 & 60.86 \\ 
% 		ResNet-50 (Lookahead) \cite{Zhang_NIPS_2019} & 23M & 75.49 & 92.53 \\ 
% 		EfficientNet-B2 (RMSProp) 350 epochs \cite{Tan_ICML_2019}~~~ & 9.2M & 80.30 & 95.00 \\ \midrule
% 		ResNet-50 (RMSProp) reproduced & 23M & 76.48  & 93.08 \\
% 		ResNet-50 (RMSProp) GAL & 23M + 0.90M &  &  \\ \midrule
% 		MobileNetV3 (RMSProp) reproduced & 5.4M &  &  \\
% 		MobileNetV3 (RMSProp) GAL & 5.4M + 0.90M &  &  \\ \midrule
% 		EfficientNet-B2 (RMSProp) reproduced & 9.2M & 77.93$\pm$0.09 & 93.92$\pm$0.03 \\
% 		EfficientNet-B2 (RMSProp) GAL & 9.2M + 0.90M & 78.10$\pm$0.06 & 93.94$\pm$0.06 \\
		\bottomrule	
	\end{tabular}}
\end{table*}

\begin{table*}[!t]
    \centering
    \caption{Regression performance on the Boston housing \cite{Harrison_JEEM_1978}, diabetes \cite{Dua_Report_2019}, and California housing \cite{Pace_SPL_1997} dataset. 
    % Three common metrics, that is, mean absolute error (MAE), mean squared error (MSE), and coefficient of determination ($R^2$), are used for evaluation. 
    $\uparrow$ (resp. $\downarrow$) indicates that a larger (resp. smaller) score suggests better performance. The experiments are run 5 times with different random seeds. We also include the analysis of two-sample t-test on the performance of the baseline and the performance of the proposed method to measure the improvement. $t_{stat}$ and $p$ are t-statistics and p value of the t-test, respectively.}
    \adjustbox{width=0.9\textwidth}{
\textcolor{black}{
	\begin{tabular}{L{17ex} L{10ex} C{20ex} C{20ex} C{27ex}}
		\toprule
		Dataset & Method & Mean Absolute Error (MAE)$\downarrow$ & Mean Squared Error (MSE)$\downarrow$ & Coefficient of Determination ($R^{2})\uparrow$ \\
		\cmidrule(lr){1-1} \cmidrule(lr){2-2} \cmidrule(lr){3-3} \cmidrule(lr){4-4} \cmidrule(lr){5-5}
		\multirow{3}{*}{Boston housing} & Baseline & 3.9535$\pm$0.4307  & 23.0956$\pm$4.1695  & 0.7668$\pm$0.0420  \\
		 & Proposed & \textbf{2.8079}$\pm$0.2720 & \textbf{12.8808}$\pm$2.0446 & \textbf{0.8699}$\pm$0.0206  \\ \cmidrule(lr){2-2} \cmidrule(lr){3-3} \cmidrule(lr){4-4} \cmidrule(lr){5-5}
		  & $(t_{stat}, p)$ & (5.02, 1.02e-03) & (4.91, 1.17e-03) & (-4.92, 1.16e-03) \\ \midrule
		\multirow{3}{*}{Diabetes} & Baseline & 44.3832$\pm$0.7752 & 3226.0238$\pm$38.8293 & 0.3821$\pm$0.0074  \\
		 & Proposed & \textbf{41.6186}$\pm$0.2989  & \textbf{2961.5520}$\pm$29.4521  & \textbf{0.4327}$\pm$0.0056  \\ \cmidrule(lr){2-2} \cmidrule(lr){3-3} \cmidrule(lr){4-4} \cmidrule(lr){5-5}
		   & $(t_{stat}, p)$ & (7.43, 7.34e-05) & (12.13, 1.97e-06) & (-12.14, 1.95e-06) \\ \midrule
		\multirow{3}{*}{California housing} & Baseline & 1.0910$\pm$0.1297 & 2.1168$\pm$0.3629 & -0.6084$\pm$0.2757  \\
		 & Proposed & \textbf{0.7780}$\pm$0.0262 & \textbf{1.1635}$\pm$0.0655 & \textbf{0.1158}$\pm$0.0498  \\ \cmidrule(lr){2-2} \cmidrule(lr){3-3} \cmidrule(lr){4-4} \cmidrule(lr){5-5}
		   & $(t_{stat}, p)$ & (5.28, 7.43e-04) & (5.77, 4.15e-04) & (-5.78, 4.14e-04) \\ 
		\bottomrule	
	\end{tabular}
}
}
    \label{tab:regr}
\end{table*}

% %---------------------
\subsection{Performance}
\label{sec:perf}

Experimental results on CIFAR-10/100, ImageNet, and MS COCO are reported in \ref{tbl:cifar10}, \ref{tbl:cifar100}, \tabref{tbl:imgnet}, and \ref{tbl:mscoco}, respectively.
As shown in \tabref{tbl:cifar10} and \ref{tbl:cifar100}, the proposed GAL is able to work with various optimization methods, \ie SGD, Lookahead, and Adabound, to improve the performance.
Also, \tabref{tbl:imgnet} shows that it is able to work with different models and provides a performance gain.
The consistent improvement in object detection can be observed in \tabref{tbl:mscoco} on MS COCO.
% The propose GAL yields the adjusted gradients that outperform the gradients computed with state-of-the-art models and optimizers. 
Overall, the proposed GAL improves the convergence of the training process to achieve better accuracies than the standard process with various models on both tasks, which is aligned with the implication of Theorem~\ref{thm:error}. 
To further evaluate the proposed method, we apply it to the regression task. Specifically, the proposed method is applied on three regression datasets, \ie Boston housing \cite{Harrison_JEEM_1978}, diabetes \cite{Dua_Report_2019}, and California housing \cite{Pace_SPL_1997}. Three widely-used metrics, \ie mean absolute error (MAE), mean squared error (MSE), and coefficient of determination ($R^2$), are used to evaluate the performance. $R^2$ measures the accuracy and efficiency of a model on the data and is a popular metric for regression. A larger $R^2$ score indicates better performance in the regression task, while smaller MAE or MSE scores indicate better performance.
Experimental results are reported in \tabref{tab:regr}. The proposed method improves the performance on all three metrics.
To further understand the statistical significance of efficacy of the proposed method, we perform a two-sample t-test on the results of the baseline and the ones of the proposed method. 
According to the $p$ values, the results yielded by the proposed method are statistically significantly from the ones yielded by the baseline with a significance level lower than 0.05.

\begin{figure*}[!t]
    % \captionsetup[subfigure]{font=tiny,labelfont=tiny}
	\centering
    % \captionsetup[subfigure]{farskip=1pt,captionskip=1pt}
	\subfloat{%
		\includegraphics[width=0.25\linewidth]{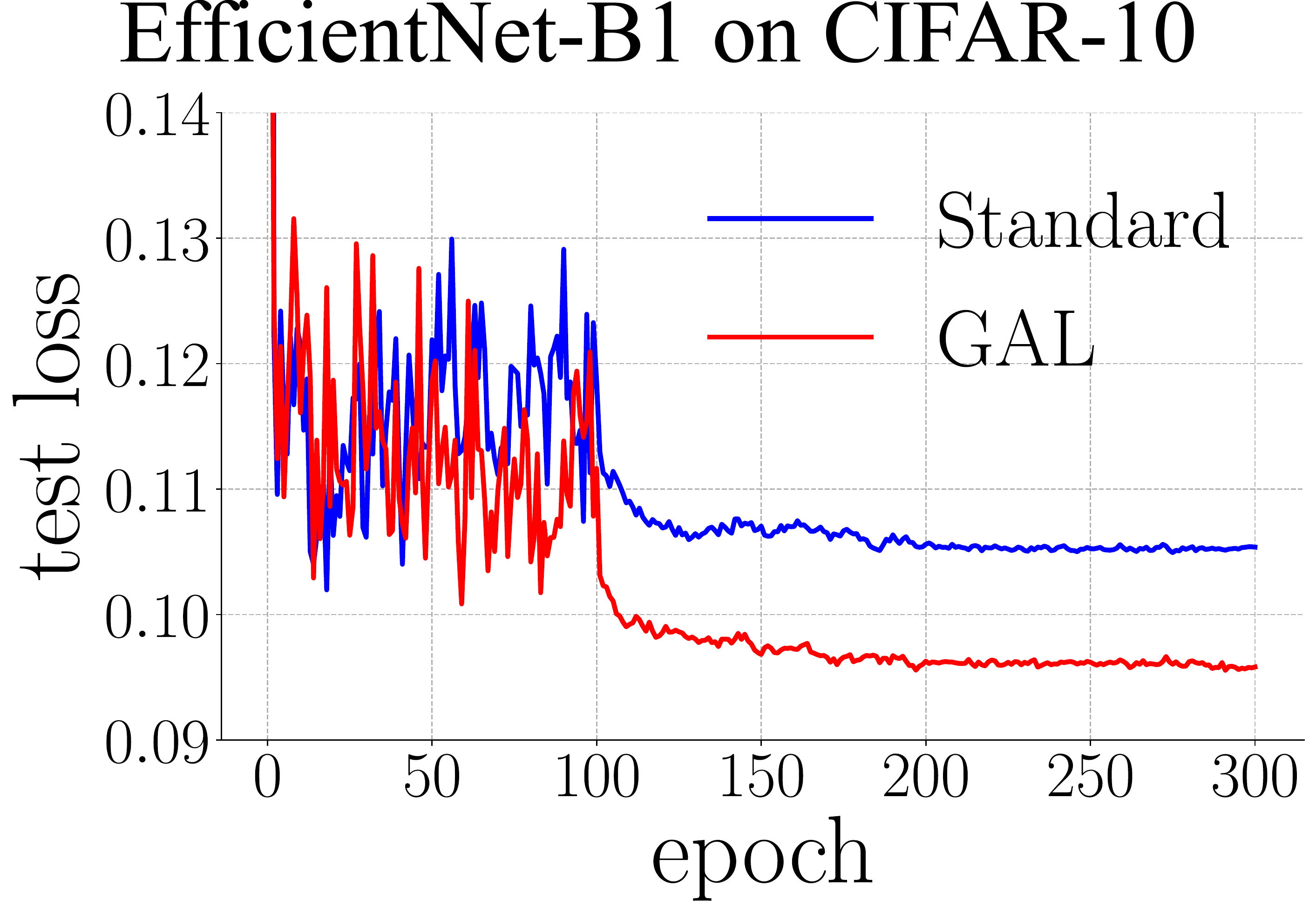}%
		\label{fig:loss_cifar10}%
	}%
	\subfloat{%
		\includegraphics[width=0.25\linewidth]{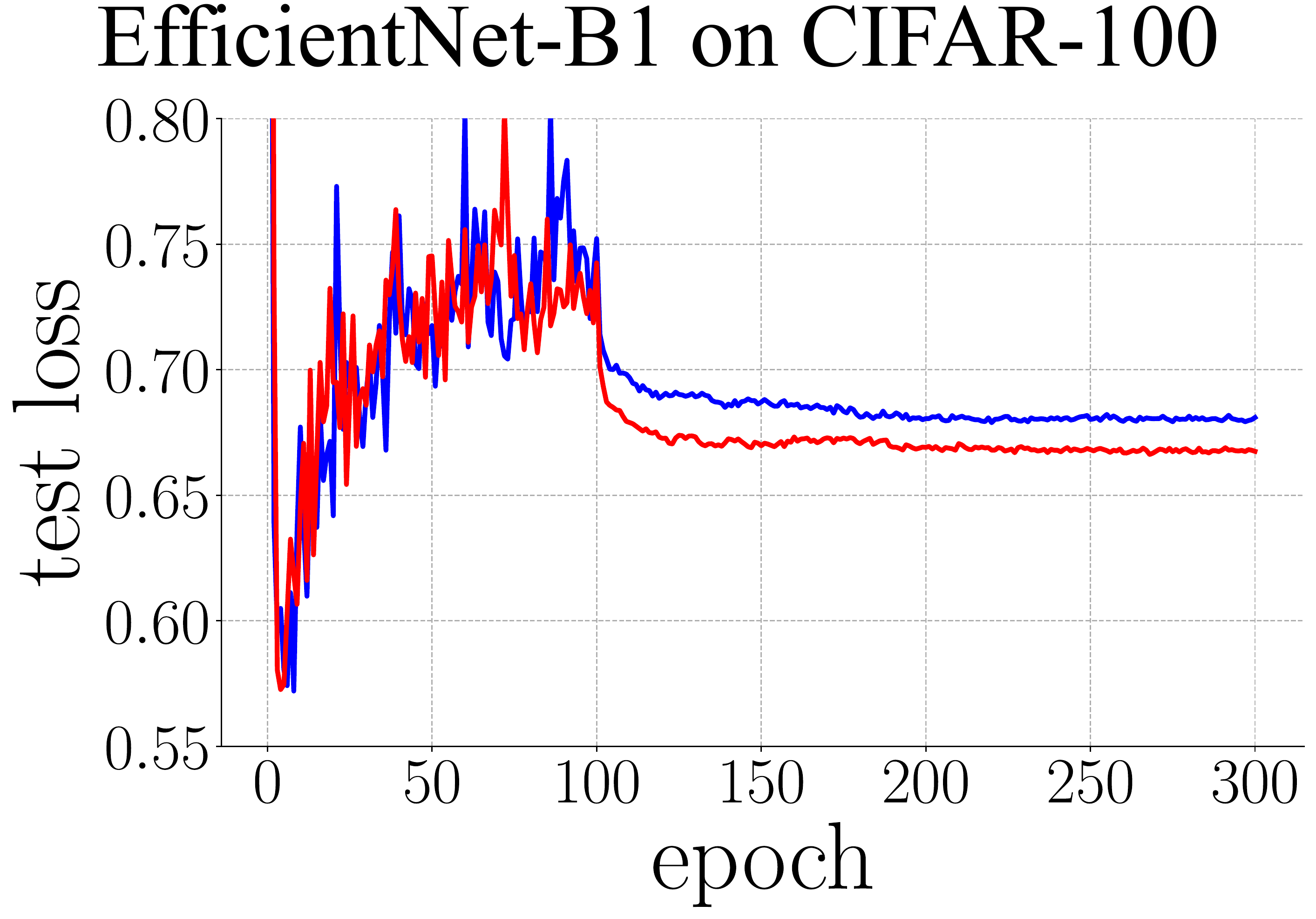}%
		\label{fig:loss_cifar100}%
	}%
	\subfloat{%
		\includegraphics[width=0.25\linewidth]{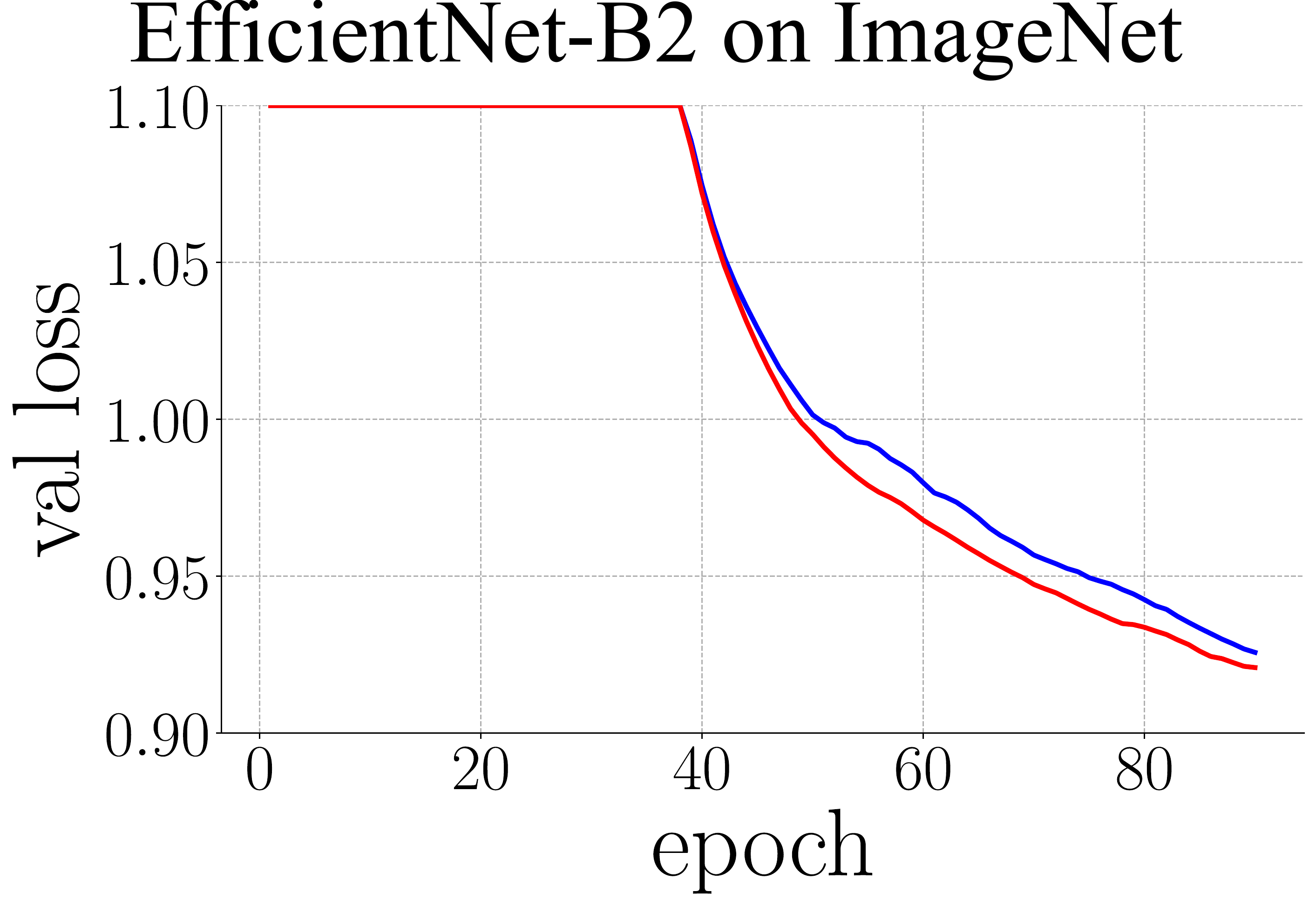}%
		\label{fig:loss_imgnet}%
	}%
	\subfloat{%
		\includegraphics[width=0.25\linewidth]{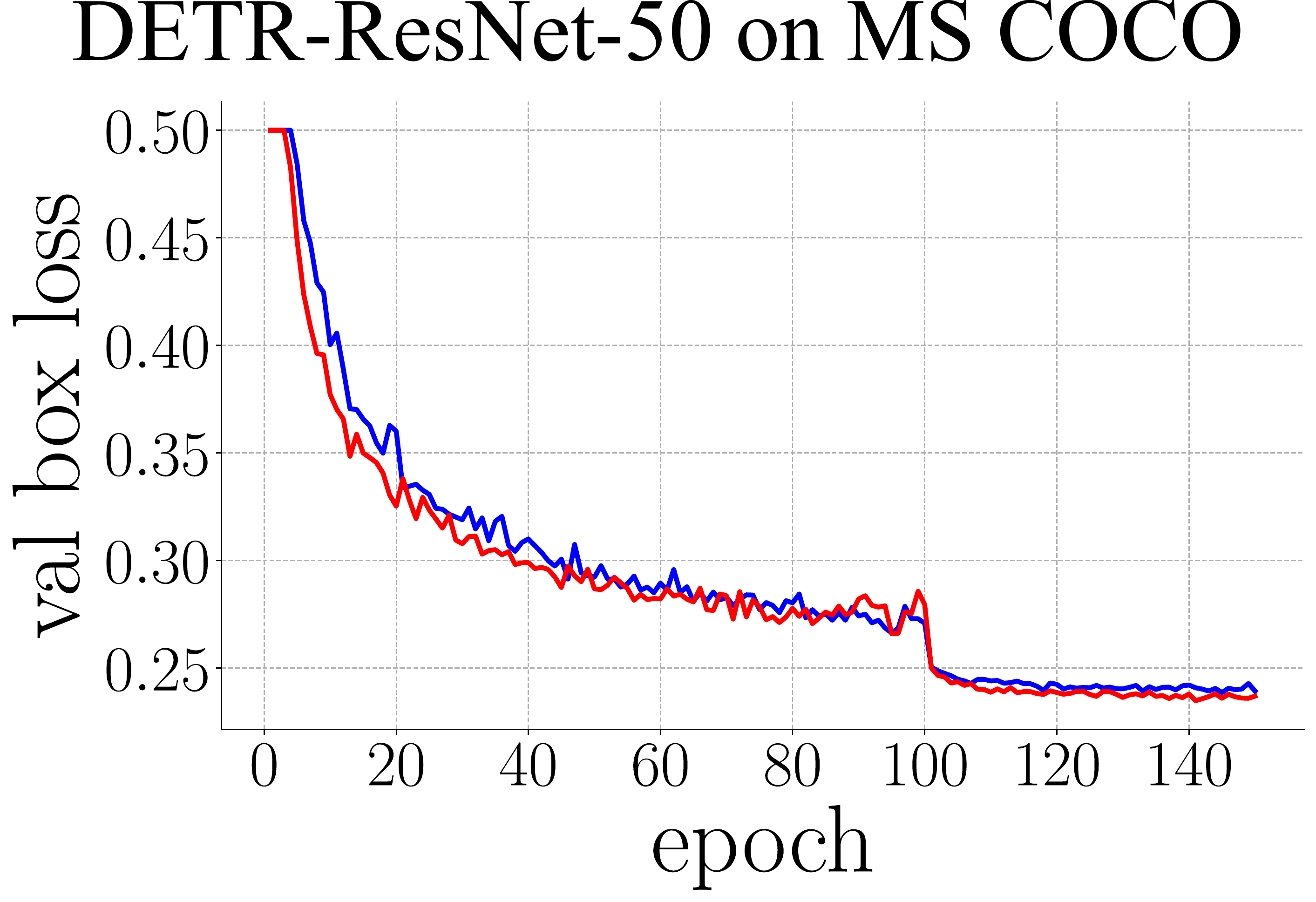}%
		\label{fig:loss_ptb}%
	}
% 	\vspace{-1em}
	\caption{
% 	Validation/test loss curves (top) and corresponding remainder curves (bottom).
% 	The columns from left to right are correlated to CIFAR-10, CIFAR-100, ImageNet, and PTB, respectively.
	Validation/test loss curves on various datasets.
% 	The maximum of y-axis is set to highlight the difference. % between the two curves.
% 	The losses are yielded by EfficientNet-B1 on CIFAR-10 (top-left) and CIFAR-100 (top-right), by EfficientNet-B2 on ImageNet (bottom-left), and by DETR on COCO (bottom-right). 
	}\label{fig:loss}
\end{figure*}

%-----------------
\section{Analysis}

\subsection{Generalization Ability and Approximation Remainder}
\label{sec:generalization}
To check the generalization ability of the models trained with GAL, we plot the loss curves on all validation (or test) set in \figref{fig:loss}. 
The losses of the models learned with adjusted gradients are lower than that of the models using vanilla gradients. 
This implies that the adjusted gradients are better than the vanilla gradients in terms of the generalizability.
% the generalization ability.

\figref{fig:remainder} shows the corresponding remainder computed by \eqn~(\ref{eqn:loss_remainder}) and the cosine similarities between vanilla gradients and adjustment vectors on ImageNet. 
Positive similarities implies that the direction of adjustment vectors has overall smaller angle with vanilla gradient (\ie~smaller than 90$^{\circ}$).
Overall, the proposed adjusted gradients converge to the local minimum more efficiently than the vanilla gradients on all datasets.
Note that there is a warm-up in ImageNet training which cause a series of fluctuations at the early epochs, but it stabalizes after 20th epoch.

\begin{figure}[!t]
	\centering
    % \captionsetup[subfloat]{farskip=1pt}
	\subfloat{%
		\includegraphics[width=0.5\linewidth]{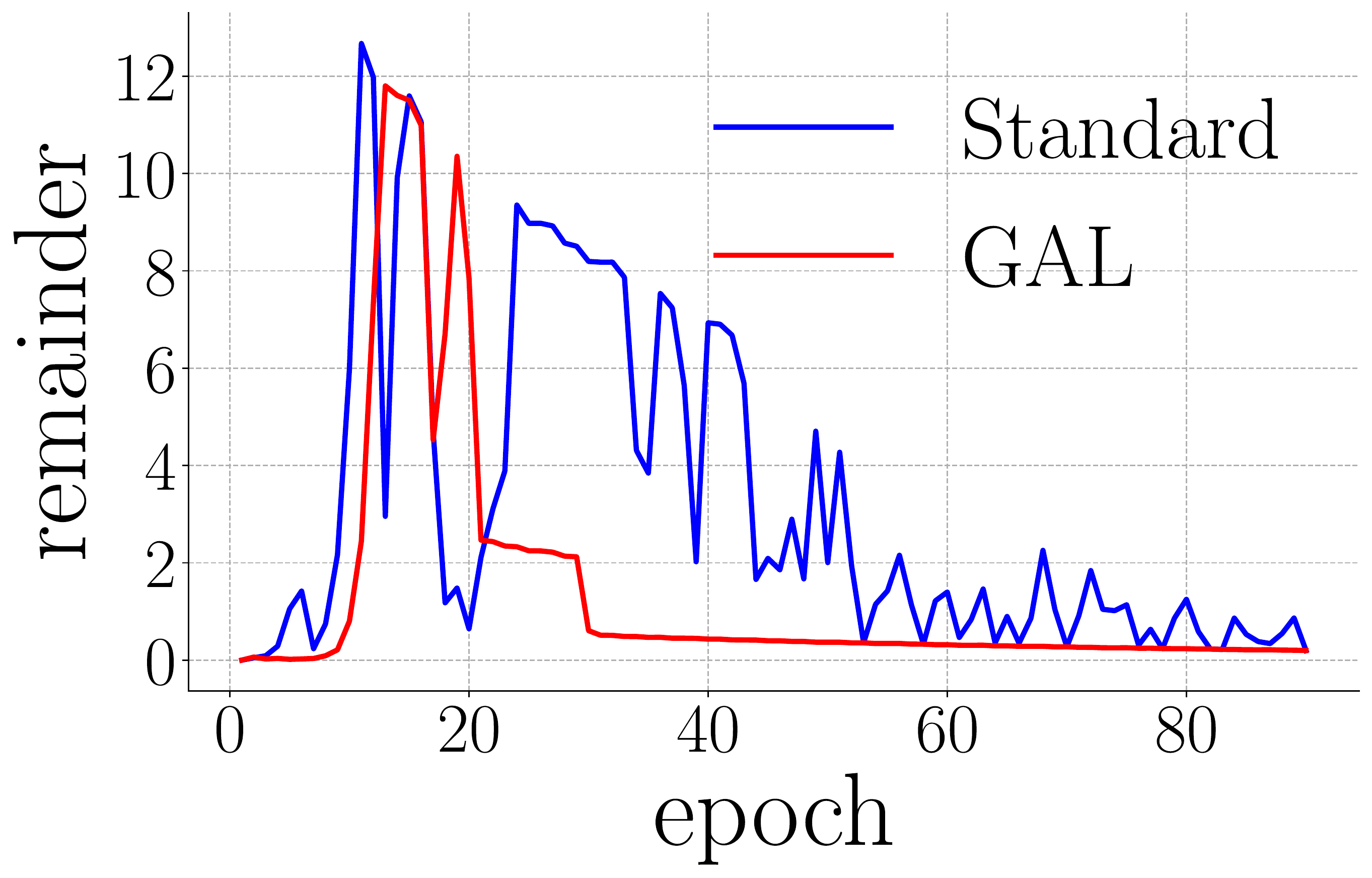}%
		\label{fig:remainder_imgnet}%
	}% 
% 	\\
	\subfloat{%
		\includegraphics[width=0.5\linewidth]{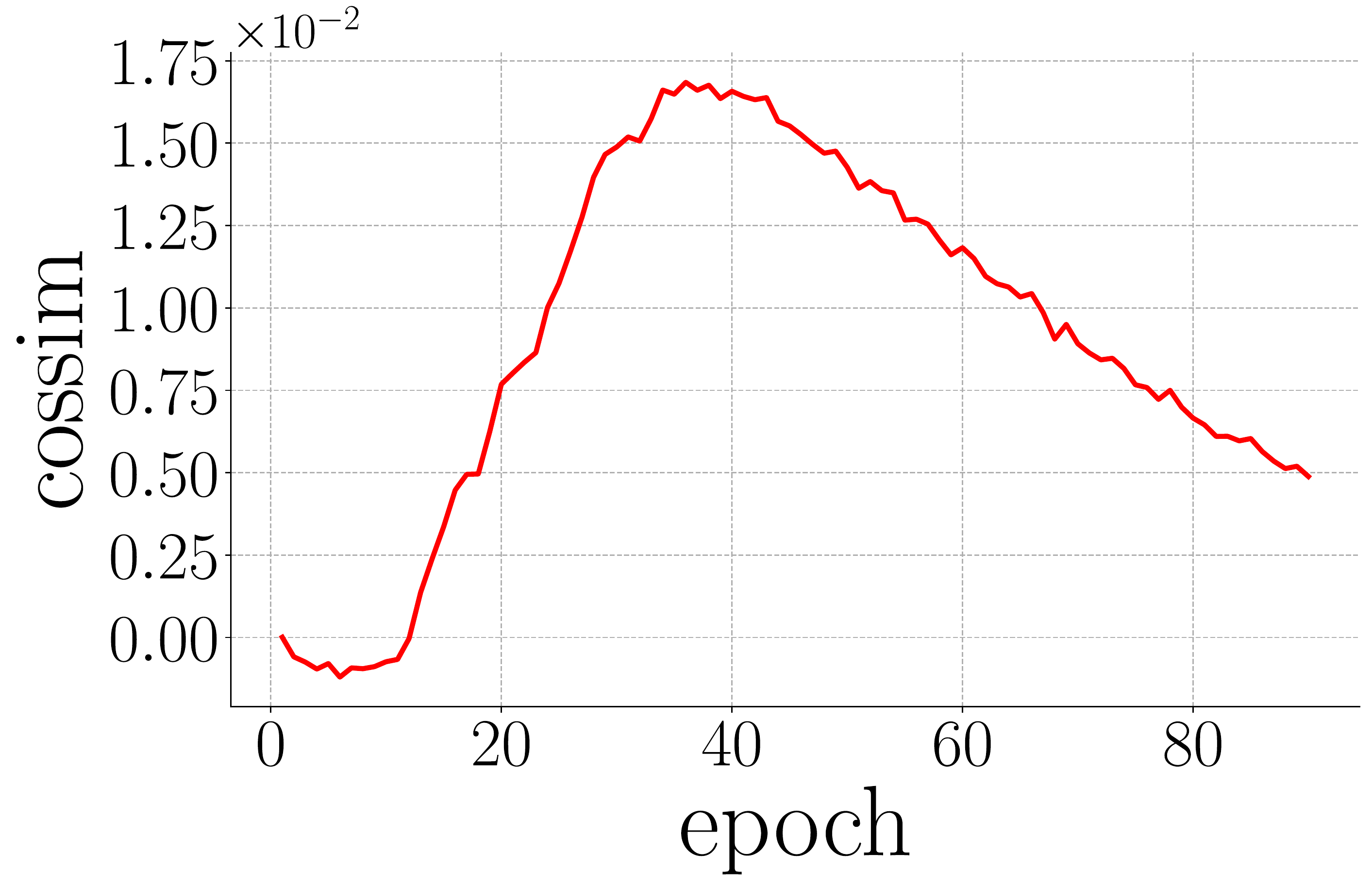}%
		\label{fig:remainder_ptb}%
	}
	\caption{
% 	Validation/test loss curves (top) and corresponding remainder curves (bottom).
% 	The columns from left to right are correlated to CIFAR-10, CIFAR-100, ImageNet, and PTB, respectively.
	Remainder curve (left) and cosine similarity curve (right) on ImageNet with EfficientNet-B2. 
	}\label{fig:remainder}
\end{figure}

\begin{figure*}[!t]
    \centering
        % \begin{minipage}{1.0\columnwidth}
	        \includegraphics[width=0.30\textwidth]{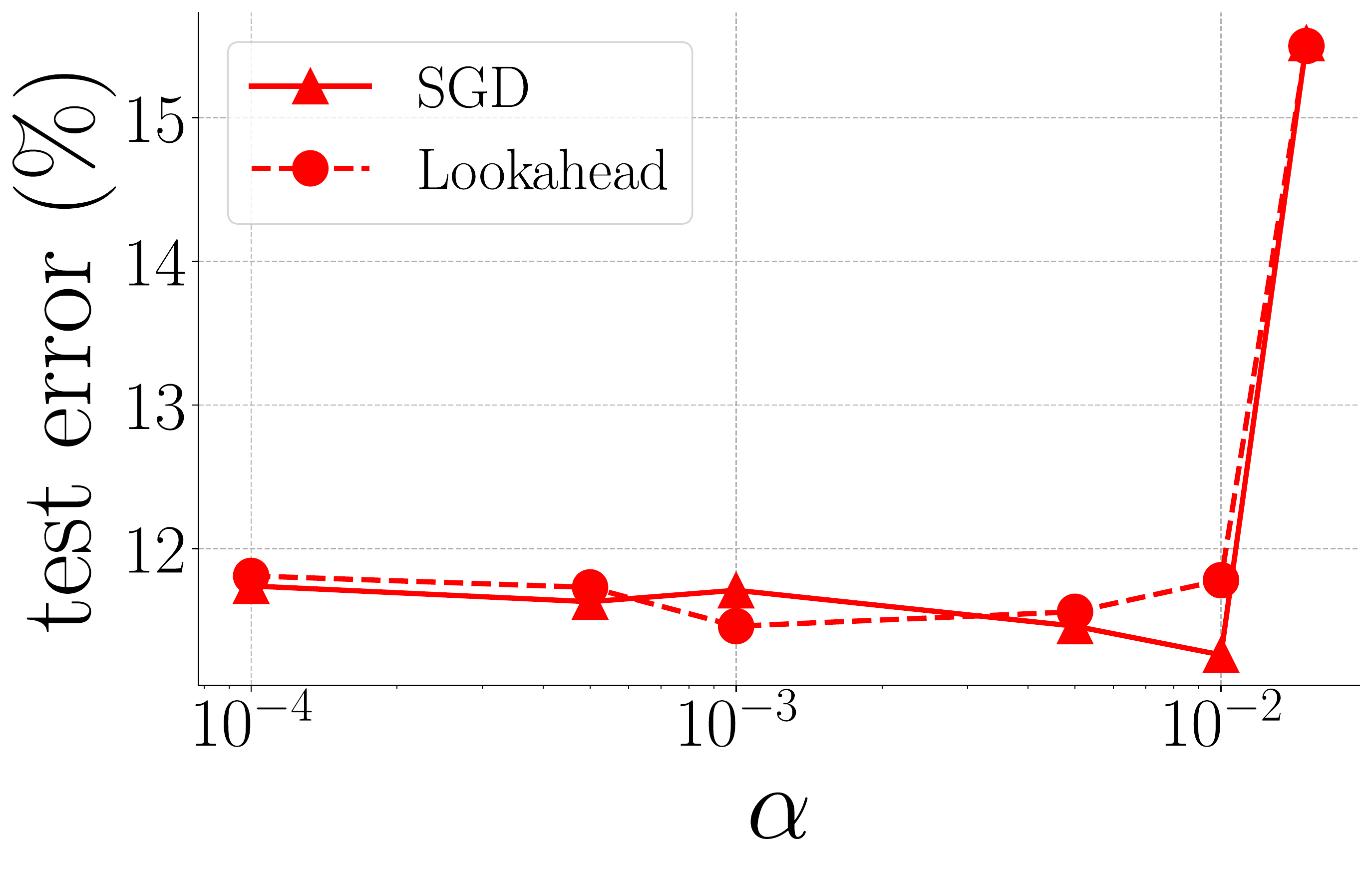} 
        	\hfill
        	\includegraphics[width=0.30\textwidth]{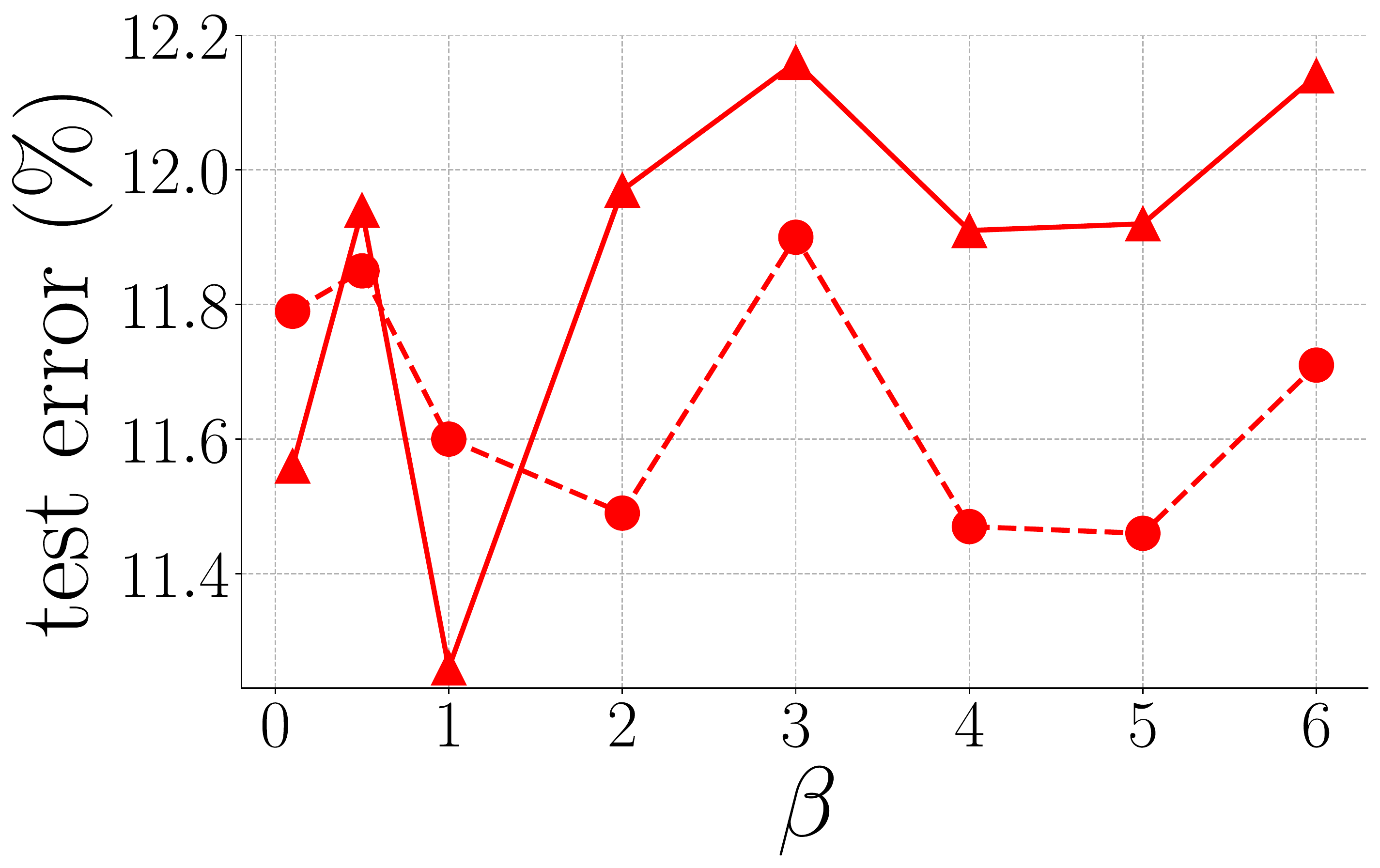}
        	\hfill
        	\includegraphics[width=0.30\textwidth]{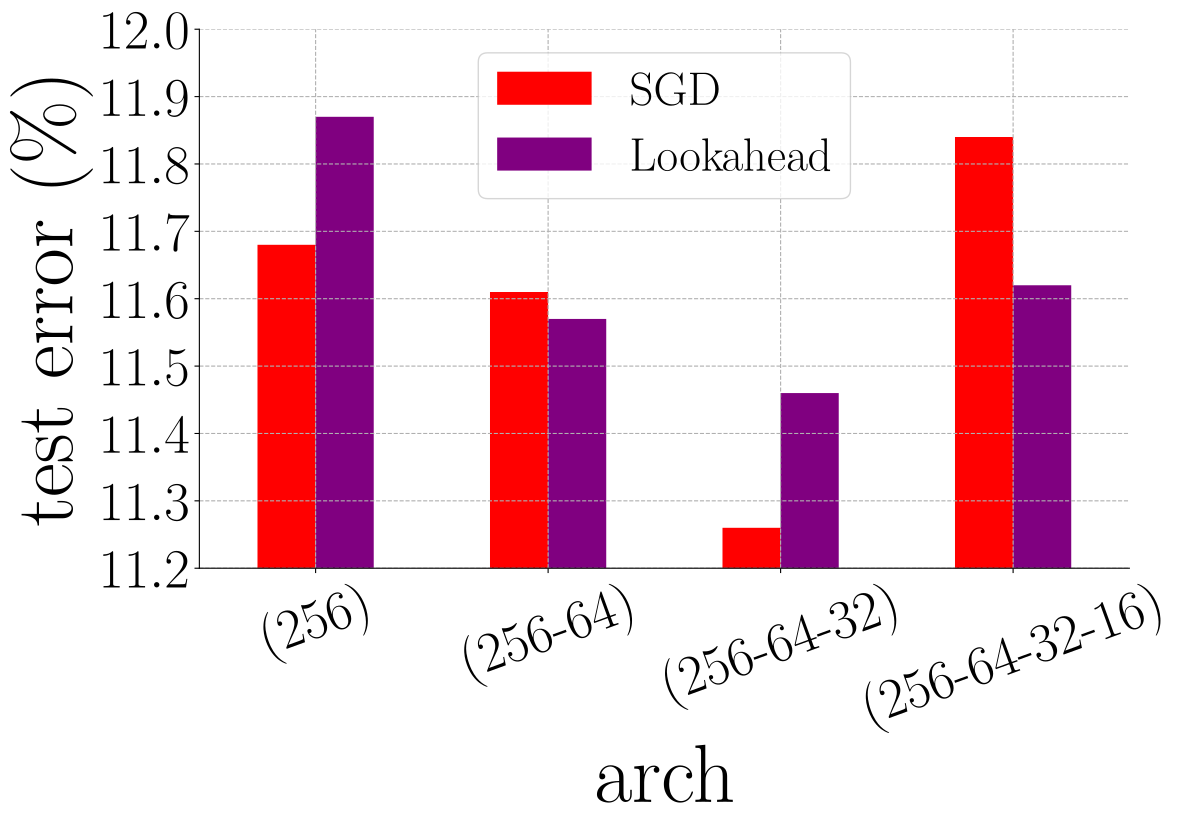}
	   % \end{minipage}
    \caption{Effects of $\alpha$ (left), $\beta$ (middle), and architecture (right) on CIFAR-100. When varying with $\alpha$, $\beta=1$ (resp. $\beta=5$) with SGD (resp. Lookahead). When varying with $\beta$, $\alpha=0.01$ (resp. $\alpha=0.001$) with SGD (resp. Lookahead). When using different architectures, (\ie (256), (256-64), (256-64-32), and (256-64-32-16)), $\alpha=0.01$ and $\beta=1$ (resp. $\alpha=0.001$ and $\beta=5$) with SGD (resp. Lookahead).}
    \label{fig:analysis_hyp}
\end{figure*}

\subsection{Effects of Random Noise}
\tabref{tbl:noise} shows the effect of random noise in the training of EfficientNet on CIFAR-100. 
The random noise are generated by a uniform or normal distribution to replace the proposed adjustment by \eqn~(\ref{eqn:predict}). 
Note that {\small $\alpha \Big| \frac{\partial \ell}{\partial z} \Big| v/ |v|$} is part of the proposed learning approach (see \eqn~(\ref{eqn:nml})).
The results shows that normalizing the adjustment vector to an appropriate range is definitely required.
This is because the gradient is subtle and sophisticated and a large adjustment vector could lead to a divergence in training.
Moreover, properly injecting some random noise using the proposed approach (see \eqn (\ref{eqn:nml}) and (\ref{eqn:adj})) improves the performance.
Yet, the noise is still less effective than the adjustment vector generated by GAL.

%--------------------------------
\subsection{Training Time}
\label{sec:train_time}
To understand the computation overhead,
we report the training time of using the baseline and the proposed method in \tabref{tab:train_time}.
In the ImageNet experiment, the learning process without the proposed GAL takes 0.3907 seconds per image to train the model,
and takes 0.4027 seconds per image with the proposed GAL. 
The extra time (\ie~12 milliseconds) w.r.t. the proposed method is used for the forward and backward process.
Similarly, the proposed method take extra 15 (11) milliseconds for training on CIFAR-10 (CIFAR-100).
Note that the experiments on CIFAR-10 and CIFAR-100 are run on a workstation equipped with 4 NVIDIA 2080 Ti GPUs, while the experiments on ImageNet are run on a workstation equipped with 8 NVIDIA V100 GPUs.

\begin{table}[!t]
	\centering
	\caption{\label{tbl:noise}
	    Effects of random noise generated by a uniform or normal distribution on the training of EfficientNet with SGD on CIFAR-100. The error rate of the standard learning process is 11.81\% while that of GAL is 11.37\%.
	    %, as reported in \tab \ref{tbl:cifar100}.
	}
% 	\vspace{-2ex}
	\adjustbox{width=0.9\columnwidth}{
	\begin{tabular}{L{12ex} L{25ex} C{9ex}}
	    \toprule
      $v$ & $\tilde{v}$ & error (\%) \\ \midrule
%       Standard & - & 11.81 \\
% 	  Proposed & $\alpha \, v/ |v|$ & 11.37 \\ \midrule
      $\mathcal{U}(-1,1)$ & $\alpha \Big| \frac{\partial \ell}{\partial z} \Big| v/ |v|$ \; {\small (\eqn (\ref{eqn:nml}))} & 11.62 \\ 
      $\mathcal{U}(-1,1)$ & $\alpha \, v/ |v|$ & 97.20 \\  \midrule
      $\mathcal{N}(0,1)$ & $\alpha \Big| \frac{\partial \ell}{\partial z} \Big| v/ |v|$ \; {\small (\eqn (\ref{eqn:nml}))} & 11.65 \\ 
      $\mathcal{N}(0,1)$ & $\alpha \, v/ |v|$ & 98.36 \\ 
      \bottomrule
     \end{tabular}}
\end{table}

\begin{table}[!t]
    \centering
    \caption{Training time of using the proposed method on each image, in comparison to the one of using the baseline. Note that the proposed gradient adjustment only takes place at the training phase. In other words, the test time w.r.t. the model trained with the proposed method should be identical to the one w.r.t. the model trained with the baseline method.}
    \adjustbox{width=0.9\columnwidth}{
\textcolor{black}{
	\begin{tabular}{L{18ex} L{18ex} C{10ex}}
		\toprule
		Dataset & Method & Time (s) \\
		\cmidrule(lr){1-1} \cmidrule(lr){2-2} \cmidrule(lr){3-3}
		\multirow{2}{*}{ImageNet} & Baseline & 0.3907 \\
		& Proposed & 0.4027 \\ \midrule
		\multirow{2}{*}{CIFAR-10} & Baseline & 0.5292 \\
		& Proposed & 0.5444 \\ \midrule
        \multirow{2}{*}{CIFAR-100} & Baseline & 0.5355 \\
		& Proposed & 0.5465 \\ 
		\bottomrule	
	\end{tabular}
}
}
    \label{tab:train_time}
\end{table}

\subsection{Effects of Hyperparameters}
\label{sec:hyp}
We analyse the effects of $\alpha$, $\beta$ and various GAL architectures with SGD and Lookahead on CIFAR-100. 
The performance is shown in \figref{fig:analysis_hyp}. 
The proposed GAL uses hyperparameters $\alpha=0.01$, $\beta=1$,  and architecture = (256-64-32) with SGD, and $\alpha=0.001$, $\beta=5$,  and architecture = (256-64-32) with Lookahead.
We vary one hyperparameter at a time while the other hyperparameters are kept unchanged in each plot.
As shown in the figure, the range $[0.0001, 0.01]$ of $\alpha$ consistently leads to lower classification errors. 
In contrast, classification errors are sensitive to $\beta$, which is optimizer-dependent. 
$\beta=1$ leads to the best performance with SGD, while $\beta=5$ leads to the best performance with Lookahead. 
Regarding the effects of architectures, We use architectures (256), (256-64), (256-64-32), and (256-64-32-16) with in \figref{fig:analysis_hyp} (right). 
The four architectures have 51.2K, 48.3K, 47.2K, and 46.1K parameters, respectively.
Overall, (256-64-32) gives rise to lower classification errors than the other architectures with SGD and Lookahead, while corresponding computational overhead is relatively low.

\subsection{Effects of Various Update Policies}
As introduced in Algorithm~\ref{alg:gl}, line 7-10, if the tentative loss $\ell_{\sigma} (z-\tilde{\eta}g)$ is less than or equal to the loss $\ell_{\sigma} (z)$, we use $g$ to update the gradients w.r.t. the weights according to the chain rule. 
We denote this case as \textit{line 7 ($\le$)}. 
In the standard process, $\frac{\partial \ell}{\partial z}$ is always used to update the gradients w.r.t. the weights and we denote this case as \textit{line 10}. 
We discuss two other possible update policies, \ie~always using $g$ and using $g$ if $\ell_{\sigma} (z-\tilde{\eta}g)$ is less than the loss. 
We denote these two cases as \textit{line 8} and \textit{line 7 ($<$)}, respectively. 
As shown in \figref{fig:analysis_cond}, policy \textit{line 8} outperforms \textit{line 10} but is not optimal as \textit{line 7 ($\le$)}.
This is because as the training process is close to the local minimum, the loss remainder is much smaller and \textit{line 10} would be more efficient than \textit{line 8}. 
Moreover, \textit{line 7 ($\le$)} is slightly better than \textit{line 7 ($<$)}.

\begin{figure}[!t]
	\centering
	\subfloat[Update policy]{%
		\includegraphics[width=0.38\linewidth]{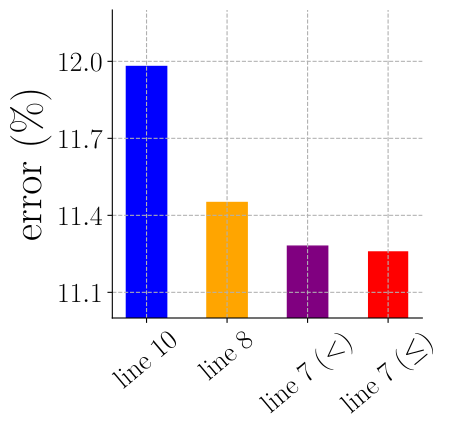}%
		\label{fig:analysis_cond}%
	}%
	\hfill%
	\subfloat[\# of improved samples]{%
		\includegraphics[width=0.58\linewidth]{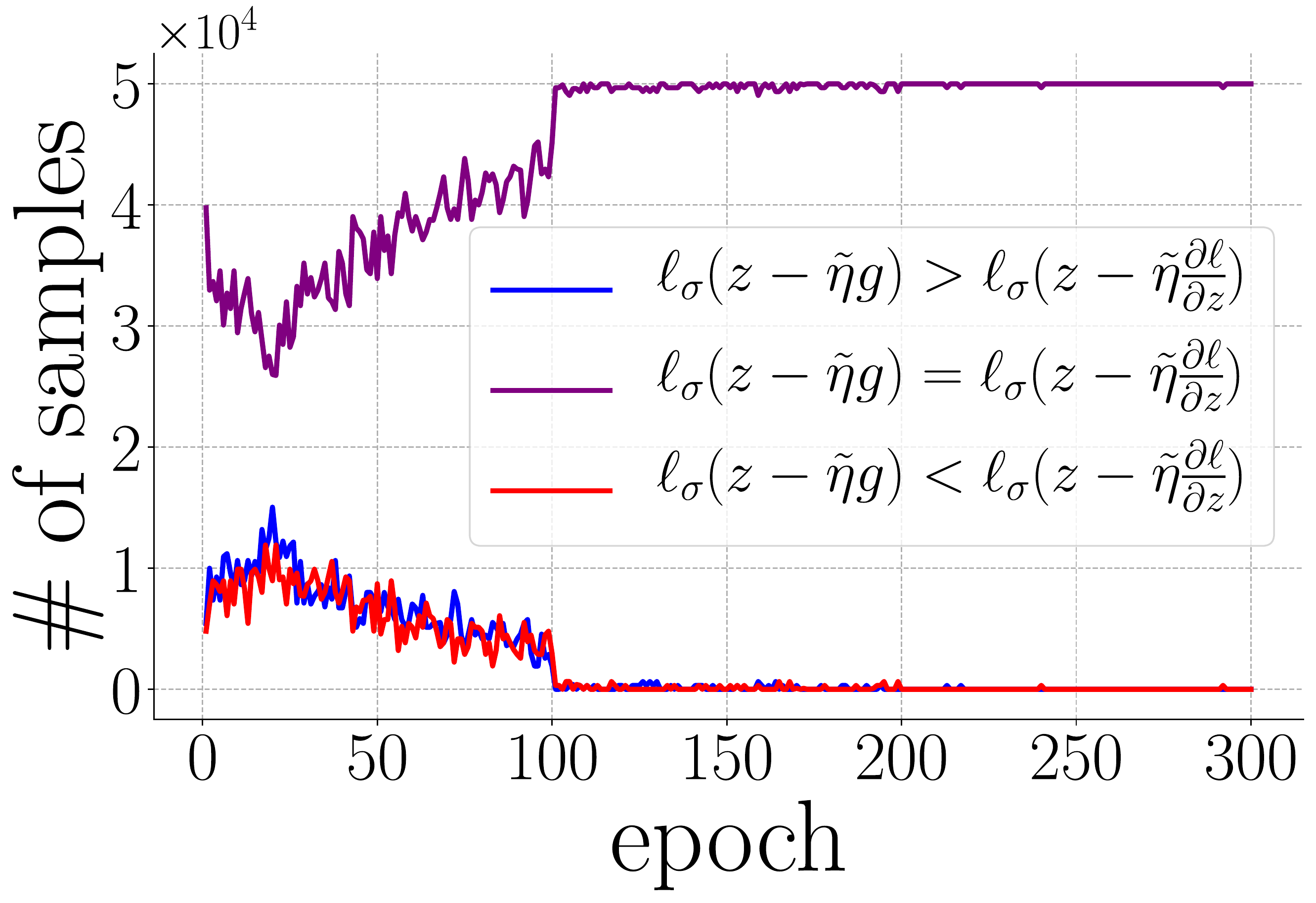}%
		\label{fig:analysis_tttloss}%
	}
	\caption{\label{fig:analysis}
    	Ablation study of the proposed GAL with SGD on CIFAR-100. 
    	(a) Effects of update policy refers to Algorithm~\ref{alg:gl} line 7-10. \textit{line x} means we use line x to generate the update, while \textit{line 7 ($<$)} means that we modify the if statement in the line 7 as $\ell_{\sigma} (z-\tilde{\eta}g) < \ell_{\sigma} (z)$. 
    	(b) Effects of adjusted gradient $g$ and vanilla gradient $\frac{\partial \ell}{\partial z}$ with tentative loss.
    	}
\end{figure}

\subsection{Adjusted Gradient vs. Vanilla Gradient}
As the proposed GAL aims to yield adjusted gradient $g$, 
it would be good to know whether $g$ leads to better descent than $\frac{\partial \ell}{\partial z}$, \ie lower loss. 
To do so, we use tentative loss to test $g$ and $\frac{\partial \ell}{\partial z}$.
\figref{fig:analysis_tttloss} shows how many times $g$ outperforms $\frac{\partial \ell}{\partial z}$ on samples. 
The results implies that GAL indeed helps adjust the vanilla gradients with tentative loss on considerable amount of samples in the early stage.

\subsection{MLPs vs. CNNs}
\label{sec:cnn}
We explore the effects of using CNNs, instead of MLPs, as the proposed gradient adjustment modules on the classification task. The results of the analysis are reported in \tabref{tab:arch_cnn}. It can be seen that CNNs have much larger numbers of parameters than MLPs (except the single layer variant), but achieve lower performance than MLPs. MLPs is a desired choice and their architectures are well aligned with the fact that the discriminative features from modern deep learning models are usually one-dimensional.

\begin{table}[!t]
    \centering
    \caption{Effects of MLPs and CNNs with SGD on CIFAR-100. In the case of CNNs, 1-d features (100) would be re-organized to 2-d features (\ie 10$\times$10), and then multiple convolutional layers with $3\times$3 kernels would be performed on the 2-d features. For example, CNN (256-64) indicates a convolutional layer with 256 3$\times$3 kernels is followed by a convolutional layer with 64 3$\times$3 kernels. Both MLPs and CNNs have a final fully-connected layer, but CNNs have an additional adaptive spatial pooling layer prior to the final layer, which reduces width and height dimensions to 1.}
    \adjustbox{width=1\columnwidth}{
\textcolor{black}{
	\begin{tabular}{L{10ex} C{18ex} C{10ex} C{10ex}}
		\toprule
		Model & Arch & Parameters & Error (\%) \\
		\cmidrule(lr){1-1} \cmidrule(lr){2-2} \cmidrule(lr){3-3} \cmidrule(lr){4-4} 
		\multirow{4}{*}{MLP} & (256) & 51.2K & 11.68 \\
		& (256-64) & 48.3K & 11.61 \\
		& (256-64-32) & 47.2K & 11.26 \\
		& (256-64-32-16) & 46.1K & 11.84 \\ \midrule
		\multirow{4}{*}{CNN} & (256) & 28.1K & 12.15 \\
		& (256-64) & 156.4K & 12.13 \\
		& (256-64-32) & 171.7K & 11.88 \\
		& (256-64-32-16) & 174.7K & 11.75 \\
		\bottomrule	
	\end{tabular}
}
}
    \label{tab:arch_cnn}
\end{table}

\section{Conclusion}
We propose a new learning approach which formulates the remainder as a learning-based problem and leverages the knowledge learned from the past approximations to enhance the learning. 
% We propose a new learning paradigm where we formulate the remainder as a learning-based problem and leverage the knowledge learned from the occurred training steps to find an effective descent. 
% We study the problem of how to leverage the knowledge learned from the occurred training steps to finding an effective descent.
% how to use the remainder of first-order Taylor's approximation to learn to adjust the gradients for boosting convergence of the training process.
To this end, we propose a gradient adjustment learning (GAL) method that employs a model to learn to predict the adjustments on gradients in an end-to-end fashion, 
which is easy and simple to adapt to the standard training process. 
Correspondingly, we provide theoretical understanding and experimental results with state-of-the-art models and optimizers in image classification, object detection, and regression tasks. 
The findings on the experimental results are aligned with the theoretical understanding on the error bound.
One intriguing extension of this work is to explore the model design to capture the subtle characteristics of gradient adjustment vectors for the adjustment prediction. 
% Moreover, a more ambitious goal is the question, can we design a model to predict more efficient gradients than the mathematical gradients?

%\appendices
%% you can choose not to have a title for an appendix
%% if you want by leaving the argument blank
%\section{}

% \section*{Acknowledgment}

% This research was funded in part by the NSF under Grants 1908711 and 1849107, and in part supported by the National Research Foundation, Singapore under its Strategic Capability Research Centres Funding Initiative. Any opinions, findings and conclusions or recommendations expressed in this material are those of the author(s) and do not reflect the views of National Research Foundation, Singapore.

\ifCLASSOPTIONcaptionsoff
  \newpage
\fi

\begin{IEEEbiography}[{\includegraphics[width=1in,height=1.25in,clip,keepaspectratio]{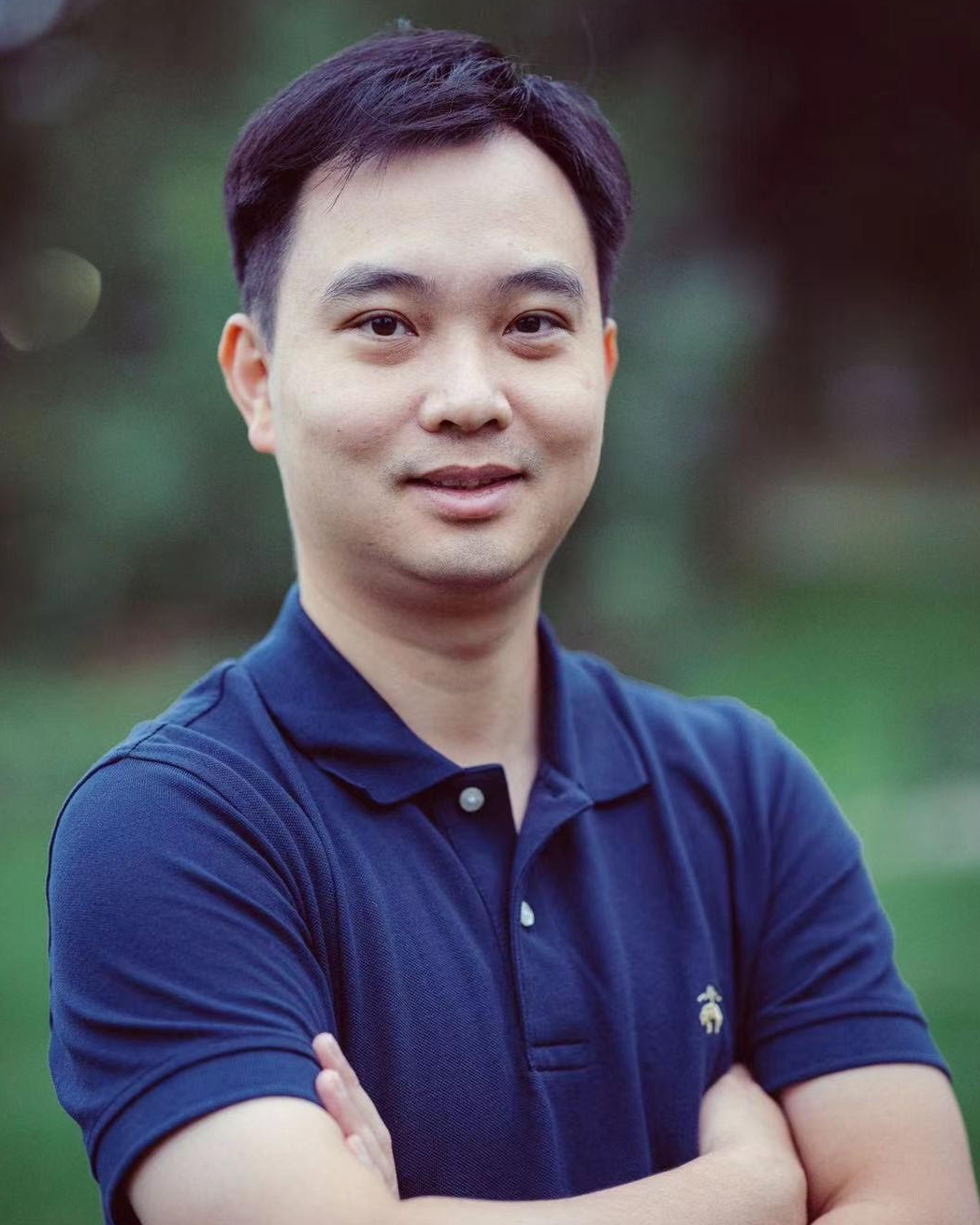}}]{Yan Luo}
	is currently pursuing a Ph.D. degree with the Department of Computer Science and Engineering at the University of Minnesota (UMN), Twin Cities. Prior to UMN, he joined the Sensor-enhanced Social Media (SeSaMe) Centre, Interactive and Digital Media Institute at the National University of Singapore (NUS), as a Research Assistant. Also, he joined the Visual Information Processing Laboratory at the NUS as a Ph.D. Student. He received a B.Sc. degree in computer science from Xi'an University of Science and Technology. He worked in the industry for several years on a distributed system. His research interests include computer vision, computational visual cognition, and deep learning. He is a student member of the IEEE since 2022.
% 	is currently pursuing the Ph.D. degree with the Department of Computer Science and Engineering, University of Minnesota at Twin Cities, Minneapolis, MN, USA. He received the B.Sc. degree in computer science from Xi'an University of Science and Technology. In 2013, he joined the Sensor-enhanced Social Media (SeSaMe) Centre, Interactive and Digital Media Institute, National University of Singapore, as a Research Assistant. In 2015, he joined the Visual Information Processing Laboratory at the National University of Singapore as a Ph.D. Student. He worked in the industry for several years on distributed system. His research interests include computer vision, computational visual cognition, and deep learning.
\end{IEEEbiography}

%\newpage

\begin{IEEEbiography}[{\includegraphics[width=1in,height=1.25in,clip,keepaspectratio]{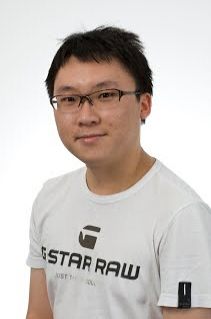}}]{Yongkang Wong}
	is a senior research fellow at the School of Computing, National University of Singapore. He is also the Assistant Director of the NUS Centre for Research in Privacy Technologies (N-CRiPT). He obtained his BEng from the University of Adelaide and PhD from the University of Queensland. He has worked as a graduate researcher at NICTA's Queensland laboratory, Brisbane, OLD, Australia, from 2008 to 2012. His current research interests are in the areas of Image/Video Processing, Machine Learning, Action Recognition, and Human Centric Analysis. He is a member of the IEEE since 2009.
\end{IEEEbiography}

\begin{IEEEbiography}[{\includegraphics[width=1in,height=1.25in,clip,keepaspectratio]{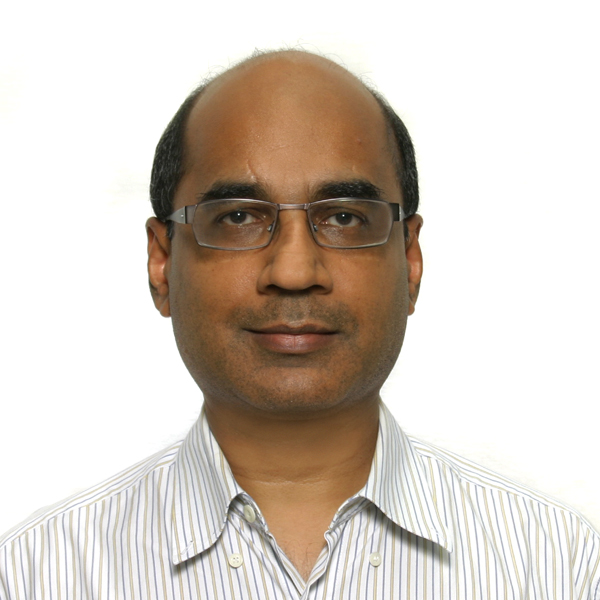}}]{Mohan S. Kankanhalli}
	is Provost's Chair Professor of Computer Science at the National University of Singapore (NUS). He is the Dean of NUS School of Computing and he also directs N-CRiPT (NUS Centre for Research in Privacy Technologies) which conducts research on privacy on structured as well as unstructured (multimedia, sensors, IoT) data. Mohan obtained his BTech from IIT Kharagpur and MS \& PhD from the Rensselaer Polytechnic Institute. Mohan’s research interests are in Multimedia Computing, Computer Vision, Information Security \& Privacy and Image/Video Processing. He has made many contributions in the area of multimedia \& vision – image and video understanding, data fusion, visual saliency as well as in multimedia security – content authentication and privacy, multi-camera surveillance. Mohan is a Fellow of IEEE.
\end{IEEEbiography}

\begin{IEEEbiography}[{\includegraphics[width=1in,height=1.25in,clip,keepaspectratio]{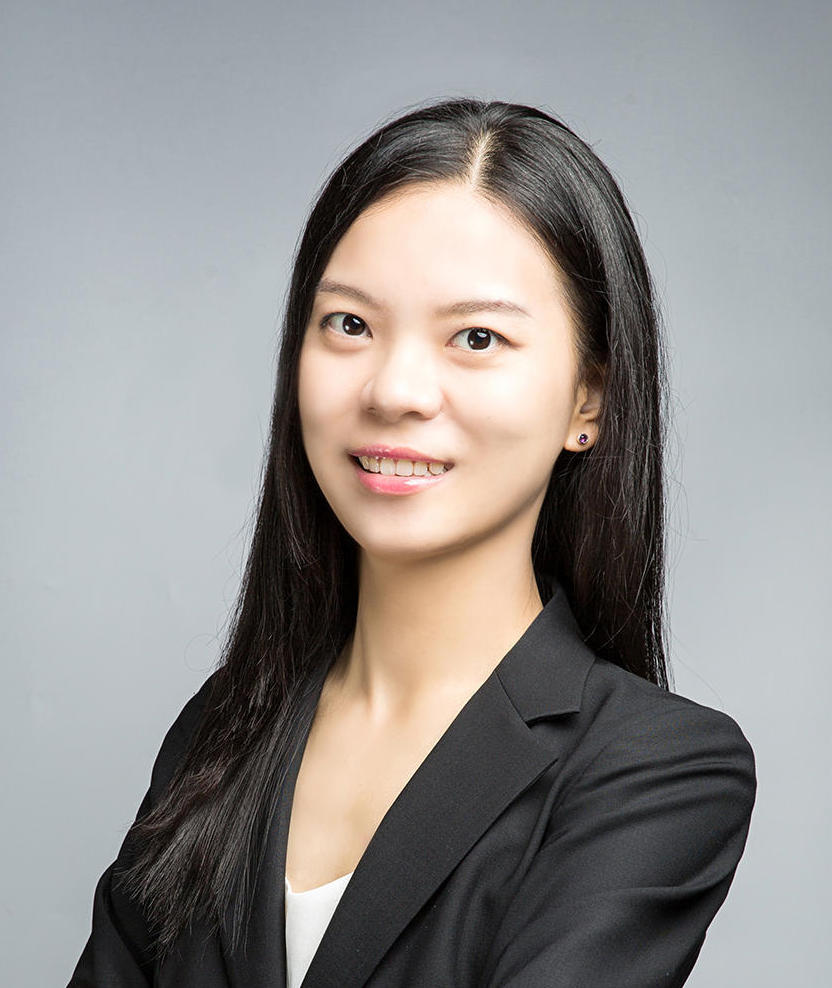}}]{Qi Zhao}
	is an associate professor in the Department of Computer Science and Engineering at the University of Minnesota, Twin Cities. Her main research interests include computer vision, machine learning, cognitive neuroscience, and healthcare. She received her Ph.D. in computer engineering from the University of California, Santa Cruz in 2009. She was a postdoctoral researcher in the Computation \& Neural Systems at the California Institute of Technology from 2009 to 2011. Before joining the University of Minnesota, Qi was an assistant professor in the Department of Electrical and Computer Engineering and the Department of Ophthalmology at the National University of Singapore. She has published more than 100 journal and conference papers, and edited a book with Springer, titled Computational and Cognitive Neuroscience of Vision, that provides a systematic and comprehensive overview of vision from various perspectives. She serves as an associate editor of IEEE TNNLS and IEEE TMM, as a program chair WACV’22, and as an organizer and/or area chair for CVPR and other major venues in computer vision and AI regularly. She is a member of the IEEE since 2004.
\end{IEEEbiography}


%\vfill


\end{document}